\documentclass{article}

\PassOptionsToPackage{square,numbers}{natbib}

\usepackage[final]{neurips_2021}

\usepackage[utf8]{inputenc} 
\usepackage[T1]{fontenc}    
\usepackage{hyperref}       
\usepackage{url}            
\usepackage{booktabs}       
\usepackage{amsfonts}       
\usepackage{nicefrac}       
\usepackage{microtype}      
\usepackage{xcolor}         
\usepackage{amsmath, amsthm}
\usepackage{bm,xspace}
\usepackage[export]{adjustbox}
\usepackage{caption}
\DeclareCaptionFormat{cont}{#1 (cont.)#2#3\par}
\usepackage{graphicx}
\usepackage{subfigure}
\usepackage{diagbox}
\usepackage{algorithm}
\usepackage{algorithmic}
\usepackage{wrapfig}
\usepackage{enumitem}

\usepackage{eqparbox}
\renewcommand{\algorithmiccomment}[1]{\hfill\eqparbox{COMMENT}{\# #1}}

\newtheorem{theorem}{Theorem}[section]

\newtheorem{definition}[theorem]{Definition}
\newtheorem{remark}[theorem]{Remark}

\newtheorem*{conjecture*}{Conjecture}
\newtheoremstyle{nonindented}{1ex}{1ex}{}{}{\bfseries}{.}{.5em}{}
\newtheoremstyle{indented}{1ex}{1ex}{\itshape\addtolength{\leftskip}{0.6cm}\addtolength{\rightskip}{0.6cm}}{}{\bfseries}{.}{.5em}{}
\theoremstyle{nonindented}
\theoremstyle{indented}
\theoremstyle{plain}

\newcommand{\union}{\cup}
\newcommand{\intersect}{\cap}

\renewcommand{\tilde}{\widetilde}

\def\min{\qopname\relax n{min}}
\def\max{\qopname\relax n{max}}

\newcommand{\RR}{\mathbb{R}}

\def\H{\mathcal{H}}

\def\N{\mathcal{N}}

\def\eps{\epsilon}

\newcommand{\one}{{\bf 1}}

\newcommand{\norm}[1]{\left\lVert#1\right\rVert}

\newenvironment{lp*}{\begin{equation*}  \begin{array}{lll}}{\end{array}\end{equation*}}

\graphicspath{ {./figures/} }

\definecolor{orange}{rgb}{1,0.5,0}
\definecolor{darkspringgreen}{rgb}{0.09, 0.45, 0.27}

\newtheorem*{theorem*}{Theorem}
\newtheorem*{proposition*}{Proposition}

\newcommand{\LSC}{LSC\xspace}
\newcommand{\mono}{\textrm{mono}}

\title{Least Square Calibration for Peer Reviews}

\author{
Sijun Tan\\
Department of Computer Science \\
University of Virginia\\
Charlottesville, VA 22903 \\
\texttt{st8eu@virginia.edu}
\And
Jibang Wu \\
Department of Computer Science \\
University of Virginia\\
Charlottesville, VA 22903 \\
\texttt{jw7jb@virginia.edu}
\And
Xiaohui Bei\\
School of Physical and Mathematical Sciences\\
Nanyang Technological University\\
Singapore 637371 \\
\texttt{xhbei@ntu.edu.sg}
\And
Haifeng Xu\\
Department of Computer Science \\
University of Virginia\\
Charlottesville, VA 22903 \\
\texttt{hx4ad@virginia.edu}
}

\begin{document}

\maketitle

\begin{abstract}
    Peer review systems such as conference paper review often suffer from the issue of miscalibration.
    Previous works on peer review calibration usually only use the ordinal information or assume simplistic reviewer scoring functions such as linear functions.
    In practice, applications like academic conferences often rely on manual methods, such as open discussions, to mitigate miscalibration.
    It remains an important question to develop algorithms that can   handle different types of miscalibrations based on available prior knowledge.
    In this paper, we propose a flexible framework, namely \emph{least square calibration} (LSC), for selecting top candidates from peer ratings.
    Our framework provably performs perfect calibration from noiseless linear scoring functions under mild assumptions, yet also provides competitive calibration results when the scoring function is from broader classes beyond linear functions and with arbitrary noise.
    On our synthetic dataset, we empirically demonstrate that our algorithm consistently outperforms the baseline which select top papers based on the highest average ratings.
\end{abstract}


\section{Introduction} \label{sec:intro}

Peer review has been essential to the advancement of academic research. In a typical conference peer review process, the submitted papers are assigned to a set of external reviewers and will receive cardinal scores from them. The final decision of which papers get accepted to the conference is usually determined by the average score received for each paper.
However, during such a process, different reviewers may have different standards and biases in their evaluations. For example, the same paper may receive a high score (e.g., 8 out of 10) from a lenient reviewer and a low score (e.g., 2 out of 10) from some strict reviewer. As a result, the average score may not always reflect the true quality of a paper. This issue is known as \emph{miscalibration}, and is a prevalent problem in many other online review systems as well such as TripAdvisor or Yelp.
In order to ensure a fair and unbiased outcome, it is crucial to the concept and to the success of a peer reviewing system to have an efficient and effective calibration algorithm.



There has been a body of works attempting to address the miscalibration issue in peer reviews. Most of these works can be categorized into two directions.
The first line of works relies on making simplifying assumptions on the form of miscalibration. For example, \cite{Paul81,Baba13,Ge11, Mackay17,Wang20Debiasing} model the miscalibration as one-parameter additive bias for each reviewer; \cite{Paul81,Roos11} assume each reviewer's scoring function to be a linear function. Many of these works are able to accurately learn the unknown parameters of the miscalibration due to its simple structure.
However, these strong structural assumptions on the miscalibration are often too strong for these results to be widely applicable.
The second line of research~\citep{Rokeach68,Freund03,Harzing09,Mitliagkas11,Negahban12,Ammar12} allows the miscalibration to be arbitrary, and only relies on ordinal ranking information obtained from cardinal scores to calibrate.
A recent work~\cite{Wang18} bridges these two lines of research by showing that even when the reviewers' functions are arbitrary, there exist cardinal score estimators that perform better than any ordinal ranking estimators. However, their work is still pertained to worst case assumptions, and is restricted to the scenario where each reviewer only reviews a pair of items. It remains an important open problem to design calibration models and algorithms that can work with weaker simplifying assumptions (e.g non-linear functions) in a practical scenario.

{\bf Summary of Contributions. }
In this work, we propose a new optimization-driven unsupervised learning framework, \emph{least square calibration} (\LSC).  As the name suggests, the framework leverages the natural and powerful least square method for handling different types of miscalibrations in peer reviews.  
A key strength of our framework is its \emph{flexibility} in incorporating various levels of prior knowledge, which consequently results in different qualities of calibration. Specifically, we capture the system's prior knowledge about  reviewers' scoring  as a \emph{hypothesis class} of scoring functions.   Depending on how much the system knows about reviewers, the hypothesis class can be as ``rough'' as all monotone functions, or as specific as   linear functions,   convex/concave functions, Lipschitz continuous functions,  or even their mixtures.  Our framework can be tailored to accommodate all these hypothesis classes and moreover, cope with the reviewers' noisy perception of true item qualities.
 

We demonstrate the effectiveness of the \LSC framework both theoretically and empirically.  We start our theoretical analysis  from the hypothesis class of \emph{linear} scoring functions. When reviewers' scores are generated by \emph{arbitrary} linear functions of true item qualities  \emph{without} noise, our main result is a necessary and sufficient characterization about when LSC can (essentially) perfectly recover true item qualities. Interestingly, it turns out that perfect recovery is  fully captured by a novel  graph connectivity notion for a carefully constructed reviewer graph, which we coin \emph{double connectivity}. This result reveals useful insights about what kind of reviewer-item assignments may be good for calibration, which is further justified in our experiments. When reviewers perceive item qualities with noises, we show how to reduce the LSC framework to a tractable convex optimization program when scoring functions are linear. We then extend this technique to accommodate other function classes  including monotone, convex, or Lipschitz continuous functions.  Our theoretical analysis is concluded with a discussion about the connection of LSC to other  well-known  problems such as regression  and matrix seriation.  

Empirically,  we test the performance of LSC via extensive experiments. We show that LSC consistently outperforms previous methods in various problem setups. The experiments also demonstrate that the calibration quality of LSC improves as our prior knowledge improves and LSC  is robust to reasonable amount of misspecified prior knowledge. 
We hope these encouraging findings can motivate more future studies about how to identify more accurate class   of  reviewers' scoring functions.

\section{Problem Formulation}
In this section, we will first formally describe the peer review problem and then introduce our   calibration approach.  We use conference peer review as an illustrative example, though our model and methods can be trivially generalized to many other peer review systems.

{\bf The Peer Review Problem. } In a peer review process such as an academic conference, there are $N$ items and $M$ reviewers. By convention, we use $[N]=\{ 1,\cdots, N\}$ to denote the set of items and $[M] = \{ 1, \cdots, M\}$ to denote the set of reviewers. We assumes that each item has an inherent ground truth quality $x^*(i)$. Throughout the paper, the superscript notation ``$*$''  will always be used to denote \emph{true} qualities. For simplicity, we assume that there are no ties and the item qualities are all distinct from each other. 
Each reviewer $j$ reviews a subset of the items $I_j \subseteq [N]$. For now, we allow $I_j$ to be an arbitrary subset, which can even be a singleton or the entire set of items $[N]$. We will specify later under what conditions about $\{ I_j\}_{j\in[M]}$ our algorithms will work.   
We denote $I^\ell_j \in [N]$ as the \emph{index} of the $\ell$'th item reviewed by reviewer $j$.


{\bf Scoring Functions and Score Generations. } \label{sec:score-generate} Each reviewer $j\in [M]$ is characterized by a monotone \emph{scoring function} $f_j: \RR \to \RR$.
Like most previous works in this space  \cite{Ge11,Roos11,Baba13,Mackay17,Wang20Debiasing}, we assume that reviewers' scores are generated by applying her scoring function to a \emph{noisy} perception of the item's true quality. Formally, reviewer $j$'s score about item $I_j^\ell$, denoted by   $y_j^\ell $, is  generated as follows
\begin{equation}\label{eq:score-generate}
 \text{Score Generation: } \quad y_j^\ell  := f_j (x^* ({I_j^\ell}) + \epsilon_j^\ell)
\end{equation}
where $f_j$ is reviewer $j$'s scoring function and $\epsilon_j^\ell$ is an independent zero-mean  noise. 
For convenience, we always assume the scores for each reviewer $j$ is ordered increasingly such that $y_j^1 < y_j^2 \cdots < y_j^{|I_j|}$.

{\bf Design Goal. }
We are interested in designing an item selection algorithm.
The algorithm takes as input the item assignment $\{I_j\}_{j \in [N]}$ for a set of items $[M]$ and a set of reviewers $[N]$,    the reviewers' scores for each of their assigned item $\{y_j^\ell\}_{j \in [M], \ell \in [I_j]}$, and an additional threshold parameter $n \leq N$. The goal of the algorithm is to output a set $S$ of $n$ items that best matches the top $n$ items according to the (unknown) true qualities  $x^*$. In the experimental section, we will describe multiple metrics to evaluate the performance of our algorithm.





{\bf Restricted Classes of Scoring Functions. } 
Note that when defining   reviewers' scoring functions, we do not make any assumptions on the properties of these functions other than monotonicity. This is to ensure that our framework is flexible  and can handle a wide range of problems.

However, in many situations, the designer may have some additional prior knowledge
and infer more structures about the scoring functions. The reason of assuming some restricted class of scoring functions is multifold. First, such functions may be a good proxy or approximation of the true scoring functions based on our  prior knowledge. Second, it may help us to avoid the issue of \emph{overfitting}\footnote{Overfitting here is in the similar sense to degenerated cases such as setting $k=n$ in k-means clustering of $n$ samples, or in low-rank matrix completion where $k$ is the target rank, $n$ is the matrix size.} since allowing arbitrary scoring functions may be able to fit the review scores well but leads to  poor generalization and make the result prone to perception noise. 
For example, many previous works~\cite{Roos11, Ge11, Baba13, Mackay17} have focused on the basic model of \emph{linear} scoring functions, defined as follows.

\begin{definition}[Linear Scoring Function]
	A linear scoring function  of reviewer $j$ has two parameters: $a_j \geq 0$ and $b_j \in \RR$ and is of the form $f_j (\tilde{x}(=x+\epsilon)) = a_j \tilde{x} + b_j$.
\end{definition}

 In addition to linear scoring functions, one may also consider broader class of scoring functions  when less prior knowledge is available or when we want to make less stringent assumptions.  {For example, our framework can also handle scoring functions that are convex or concave or a mixture of these functions.\footnote{A  scoring function $f$ is  \textit{convex} if for any $x_1, x_2  \in \RR$ and for any $ t \in (0,1)$, we have $f(t x_1 + (1-t)x_2) \leq tf(x_1)+(1-t)f(x_2)$. Function $f$ is concave if  $-f$ is convex.}} 
 However, we emphasize that it is \emph{not} our purpose to   justify which of these classes of scoring functions are more realistic since this answer would depend on the concrete domains of the problem. Instead,  our goal is to develop a generic learning framework that can integrate any of such prior knowledge to improve the  calibration. That being said, we believe it is an important future research direction to design heuristics, regularizations and other methods to choose the hypothesis class for better calibration results as well as to understand the potential  issue of overfitting for hypothesis class that are  too broad.  
%



\section{The Optimization-Driven Framework}\label{sec:model}
Central to our approach is a framework to recover the underlying item quality $\mathbf{x}$ that can properly ``explain'' the observed review scores   given the assumptions about the scoring function structures. Note that our recovered qualities may be up to some constant shifting or re-scaling from the true qualities. However, as long as the shifting and re-scaling preserve monotonicity, this will not affect the accuracy of our selection of the best $n$ items.

Following the notation in machine learning, let $\H$ denote a \emph{hypothesis class} of scoring functions. For example, for linear scoring functions, the hypothesis class is $\H = \{f: f(x) = ax+b \text{ for some }a\geq 0, b\in \RR \}$. We denote this linear scoring function class as $\H_L$ for convenience. More generally, $\H$ may be any class of functions depending on designer's prior knowledge of these scoring functions.

Given a hypothesis class $\H$, our goal is to find scoring functions $f_1, \cdots, f_M \in \H$ for all reviewers, such that the distance between the true qualities $\mathbf{x}$ and qualities recovered by these selected scoring functions is minimized. In this paper, we use the Euclidean distance. This leads to the following general optimization-based framework, which we term \emph{Least Square Calibration (LSC)}. Note that in general LSC is a \emph{functional optimization problem (FOP)} with: (1) functional variables $\mathbf f = \{ f_j \in \H\}_{j\in[M]}$;
(2)  item quality variables $\mathbf x = \{x(i)\}_{i \in [N]}$; (3)  noise variables $\mathbf \epsilon = \{\epsilon^\ell_j\}_{j\in[M],\ell\in[|I_j|]}$. 

\rule{\linewidth}{0.75pt}
\noindent\textbf{Least Square Calibration (\LSC)}
\begin{align}\label{lp:framework}
 \min_{\mathbf{x}, \mathbf{f}, \mathbf{\epsilon}} & \quad \sum_{j=1}^M \sum_{\ell=1}^{|I_j|} (\epsilon_j^\ell)^2 \\ \nonumber
\textup{s.t. } & \quad y_j^\ell = f_j\left(x(I_j^\ell) +\epsilon_j^\ell \right) \tag*{$\forall j \in [M], \ell \leq |I_j|$}
\end{align}
\hrule
\vspace{1pt}
where the set cardinality $|I_j|$ is the number of items reviewer $j$ reviews. LSC tries to find the scoring functions and true qualities that best fit all the review scores in terms of the least square error.

\medskip

\textbf{The rationale behind \LSC} can be seen from several perspectives. First,
fix any (unknown) parameters $\mathbf{x}$ and $\mathbf{f}$, define random variable $Y^\ell_j = f_j\left(x(I_j^\ell) +\epsilon_j^\ell \right)$ with randomness inherited from the noise $\epsilon^\ell_j$. Consider the following parameter estimation problem\footnote{The parameters in this estimation problem includes functional parameters.} --- given observed realizations of $Y^\ell_j $ as $y^\ell_j$,
find parameters $\mathbf x$ and $\mathbf f$ to maximize the likelihood of the observation. It is not difficult to see that the optimal solution to \LSC is the maximum likelihood estimator, when $\epsilon_j^\ell$ is i.i.d. Gaussian noise with mean $0$. Second, it turns out that \LSC can also be viewed as a variant of the classic \emph{matrix seriation} problem.
In Section \ref{sec:related}, we will discuss this connection in more details as well as the connection of \LSC to  linear regression.

\subsection{Linear Scoring Functions with No Noise}
\LSC is a general calibration framework for recovering the true qualities of the items.   One cannot hope to solve it efficiently without any assumption on the scoring function hypothesis class,
since functional optimization problems are generally intractable. Fortunately, we will show   that for a  broad class of commonly used hypothesis classes, \LSC can be formulated as convex programs and therefore be solved efficiently. In some situations, we can even  show that LSC can provably recover the correct item qualities perfectly under mild assumptions. 

As a warm-up, we first restrict our attention to perhaps the mostly widely studied hypothesis class, i.e., linear  functions. We start with the simple situation with no  noise.  In this case,   it is not difficult to show that the following  linear feasibility problem\footnote{Here to guarantee $x(I_j^\ell)$ to be strictly larger than $x(I_j^{\ell-1})$ and we w.l.o.g. use $\geq 1$ for convenience since rescaling the solution will not change ranking.} solves the LSC Problem \eqref{lp:framework}.  We omit the proof of this claim   since a   more general result will be proved in subsection \ref{sec:linear-noise}. 

\rule{\linewidth}{0.75pt}
\noindent\textbf{\LSC with linear scoring functions and no noise}
\begin{align}\label{lp:linear}
	\min_{\mathbf{x}} & \quad 0 \\ \nonumber
	\textup{s.t. }&\quad x(I_j^\ell) - x(I_j^{\ell-1}) \geq 1 \tag*{$\forall j \in [M], \forall 2 \leq \ell \leq |I_j|$} \\\nonumber
	&\quad \frac{x(I_j^\ell) - x(I_j^{\ell-1})}{y^\ell_j - y_{j}^{\ell-1}} = \frac{x(I_j^{\ell+1}) - x(I_j^{\ell})}{y_{j}^{\ell+1} - y_{j}^{\ell}} \tag*{$\forall j \in [M], 2 \leq \ell \leq |I_j| -1 $} \\  \nonumber
\end{align}
\vspace{-9mm}
\hrule
\vspace{3pt}

With the prior knowledge that the scoring functions are all linear functions and that there is no noise, it is easy to see that the true item qualities $\mathbf{x}^*$ and scoring functions $\mathbf{f}^*$ must be a feasible solution to LP \eqref{lp:linear}. However, we cannot guarantee that LP~\eqref{lp:linear} does not have any other feasible solutions whose induced item order disagrees with $\mathbf{x}^*$.
This could happen, for example, when the reviewers do not review enough items, or their assigned items do not have enough overlap to perform a full calibration.
Therefore, it is an important question to understand under what condition   we can (essentially) recover  all item qualities using LP~\eqref{lp:linear}, in the following sense:

\begin{definition}[Perfect Recovery]
	A set of qualities $\{x(i)\}_{i \in [N]}$ is said to \emph{perfectly recover} the true qualities $\{x^*(i)\}_{i \in [N]}$ if there exists $a > 0$ and $b$ such that $x(i) = ax^*(i)+b$ for every $i \in [N]$.
\end{definition}

We remark that perfect recovery is a stronger condition than   recovering just the \emph{correct order} of the items. It additionally requires that the  quality gaps among different items are proportional to the gaps of true qualities.  One might wonder whether such a strong goal of perfect recovery may be too good to be possible. Surprisingly, our next result formally show that if the reviewer assignments satisfy some reasonable conditions, it is indeed possible to hope for perfect recovery. To introduce our result, we introduce a novel notion of a \textit{review graph} and the property of \emph{recover-resilience}.  
\begin{definition}[Review Graph and  Recovery-Resilience] \label{def:review-graph}
	A \emph{review graph} is an undirected multi-graph\footnote{A multi-graph is a graph that allows multiple edges between two vertices.} $G=(V,E)$ where $V$ is the set of reviewers, and there are $|I_i \intersect I_j|$ edges connecting every two reviewers $i, j \in V$. A review graph $G$ is said to be  \emph{recovery-resilient} if for any item assignments inducing $G$ and any   given review scores  $\{y_j^l \}_{j,l}$, any solution to LP \eqref{lp:linear} is a perfect recovery. 
\end{definition}  
That is, the review graph is a graph that connects reviewers to each other according to the number of common items that they reviewed. It turns out that the recovery-resilience of a review graph is fully characterized by a novel notion of \emph{graph connectivity}, which we coin the   \emph{double connectivity}. 

Double connectivity naturally generalizes the standard notion of graph connectivity.  Recall that a connected component of the reviewer graph is simply a subset of reviewers $S \subseteq [M]$ who connects to each other. We will refer to the items they have collectively reviewed, $C(S) = \bigcup_{i\in S} I_i$, as the items \emph{covered} by component $S$.  
A connected component of a graph can be identified through a simple \texttt{Repeat-Union} procedure, where we start by treating each node as a component and then repeatedly joining two components whenever they share at least \emph{one} edge. Our notion of  \emph{doubly-connected component} include all components that can be found using the same joining procedure, except that each time we join two components only when they share at least \emph{two} edges. A formal procedure, coined \texttt{Repeat-Union2}, is given in Algorithm \ref{alg:union} below.

\begin{algorithm}[ht]
	\caption{\texttt{Repeat-Union2}}
	\label{alg:union}
	\begin{algorithmic}[1]
		\INPUT Review graph with $N$ item, $M$ reviewer, and item assignment set $\{I_j\}_{j\in [M]}$
		\STATE $U \leftarrow \{\{1\}, \{2\}, \ldots, \{M\}\}$
		\WHILE {$\exists S_1, S_2 \in U$, s.t. $|C(S_1) \cap C(S_2)| \geq 2$ }
		\STATE $U \leftarrow U - S_1 - S_2 + (S_1 \cup S_2)$ (i.e., merge $S_1, S_2$)
		\ENDWHILE
		\STATE \textbf{return} $U$
	\end{algorithmic}
\end{algorithm}

\begin{definition}[Doubly-connected component] \label{def:double-connect}
	Any component that can be generated by Algorithm \ref{alg:union} is called a \emph{doubly-connected component}.
\end{definition}

The following theorem shows that to guarantee perfectly recovered qualities from LP \eqref{lp:linear} under linear scoring functions without noise, all we need is a \emph{doubly-connected} component that covers all items.  

\begin{theorem}\label{thm:perfect}
	A review graph $G$ is   recovery-resilient \emph{if and only if }   the review graph $G$ has a doubly-connected component $S$ that covers all items, i.e. $C(S) = [N]$.  
\end{theorem} 
\begin{proof}[Proof Sketch]
The proof is somewhat involved particularly for the necessity of double-connectivity.  We defer the full proof to Appendix \ref{append:perfect-proof} and only describe the   proof sketch here.  For the ``if'' direction, the key idea is  to use scores of at least two shared items between two reviewers  $j_1$ and $j_2$ to calibrate the other items they reviewed  in $I_{j_1}, I_{j_2}$.  After this calibration, we can effectively treat these two reviewers as a single  ``reviewer'' (formally, a doubly-connected component) 
who reviews all the items $I_{j_1} \cup I_{j_2}$. Algorithm~\ref{alg:union} follows exactly this inductive procedure by repeatedly merging two components into a single component, until there is a ``reviewer'' that reviewed the entire item set $[N]$, and these are the calibrated qualities that we want.  

The ``only if'' direction is more challenging.  Here, we have to prove that for   \emph{any}  reviewer graph that does not have a doubly-connected component covering $N$, there must exist an instance where perfect recovery is impossible. The main difficulty lies at crafting such an instance, by specifying paper assignment, true item qualities etc., for any such graph.  Our proof leverages nice properties of prime numbers to carefully construct the instance such that it will provably lead to at least two linearly independent solutions of LP \eqref{lp:linear}.  \end{proof}


\begin{remark} \label{rmk:doubly}
Theorem \ref{thm:perfect}   provides useful and practical insights for reviewer assignments in real applications. Specifically, it appears that  stronger connectivity among reviewers during review assignment, defined formally in the sense of   Definition \ref{def:double-connect} could be more helpful to calibrate their reviews. In practice, such connectivity can be  achieved as follows: (1) run Algorithm \ref{alg:union}; (2) if  the algorithm stops without a doubly connected component covering all items, then simply ask any reviewers in one component to generate two reviews for  two additional items covered by the other component so that these two components now become double-connected. This procedure is actually  common during today's academic conference review  as a way for calibration during post-rebuttal discussions. Interestingly,   Theorem \ref{thm:perfect} serves as a theoretic justification for this practice. In Appendix \ref{append:remark3.5}, we provide experiment results which further   suggest that the topological structure of double connectivity exhibits stronger robustness against perception noise. 
\end{remark}

\subsection{Linear Scoring Functions with Noise}\label{sec:linear-noise}
The noiseless case discussed in the previous section has an important feature: each reviewer's reviewing scores are monotone in the items' true qualities. This is a very useful property that could help the design of efficient calibration algorithms.
However, it is also a restricted assumption that is often violated in practice. In real-world conference reviews, it is very common to see two reviewers reviewing the same pair of items but rating them in different orders. This is typically explained by reviewers' noisy perception of the paper qualities~\cite{Ge11,Roos11,Baba13,Mackay17,Wang20Debiasing}. 
This leads to our model of perception noise in score generation procedure \eqref{eq:score-generate}.
\rule{\linewidth}{0.75pt}
\noindent\textbf{\LSC with linear scoring functions and with noise}
\begin{align} \label{lp:linear-noise}
\min_{\mathbf{x}, \mathbf{\epsilon}} & \quad \sum_{j=1}^M \sum_{l \leq |I_j|} (\epsilon^\ell_j)^2 & \\ \nonumber
\textup{s.t. } & 
\quad \widetilde x^\ell_j- \widetilde x^{\ell-1}_j \geq \frac{ y^\ell_j - y_j^{\ell-1}}{C} 
\tag*{$\forall j \in [M], 2 \leq \ell \leq |I_j|$} \\ \nonumber
& \quad \frac{\widetilde x^\ell_j- \widetilde x^{\ell-1}_j}{y^\ell_j - y^{\ell-1}_j} = \frac{\widetilde x^{\ell+1}_j- \widetilde x^\ell_j}{y^{\ell+1}_j - y^{\ell}_j} 
\tag*{$\forall j \in [M], 2 \leq \ell \leq |I_j|-1$} \\ \nonumber
& \quad \widetilde x^\ell_i = x^\ell_i + \eps^\ell_i 
\tag*{$\forall j \in [M], 1 \leq \ell \leq |I_j|$}
\end{align}
\hrule
\vspace{0pt}

As discussed before, we assume each $\epsilon^\ell_j$ is an independent  noise added to the true quality of the item.    Let $\H_{L}(C)$ denote the set of all   increasing linear scoring functions whose derivatives (whenever exists) are upper bounded by some constant $C$. More formally,  $C = \sup \{\frac{f(x_i) - f(x_j) }{x_i - x_j} | \forall x_i,x_j \}$. (In practice,  $C$ can be empirically chosen to be a large constant if prior knowledge is lacking.)    
After adding the noise to the linear constraints in LP~\eqref{lp:linear}, our LCS with linear scoring functions becomes the above convex program, whose validity is guaranteed by the following theorem. 
\begin{theorem}\label{thm:linear-noise}
    Convex Program \eqref{lp:linear-noise} is equivalent to \LSC with $\H= \H_{L}(C)$ in the following sense:
    for any optimal solution $(\mathbf{x}^*, \mathbf{\epsilon}^*)$ to Program \eqref{lp:linear-noise},
    there exists $\mathbf f^* = \{f^*_j \in \H_{L}(C)\}_{j \in [M]}$ such that $(\mathbf{x}^*, \mathbf{\epsilon}^*, \mathbf{f}^*)$ is optimal to LSC Program \eqref{lp:framework}.
\end{theorem}
Theorem \ref{thm:linear-noise}  shows that we can recover the same item order by solving $\mathbf{x}$ from \LSC or by solving CP \eqref{lp:linear-noise}. 
The proof of Theorem \ref{thm:linear-noise} relies on a general procedure to argue the equivalence between a functional optimization problem and a variable optimization problem. We defer the formal arguments to Appendix  \ref{append:proof-linear-noise}.

\subsection{Handling Other Classes of Scoring Functions}
It turns out that   the above approach in Section \ref{sec:linear-noise} can be generalized to  handle prior knowledge on various types of function shapes such as linearity, concavity and convexity, reference information (e.g., knowing the true quality of a few items), derivative information such as Lipschitz continuity,  etc. For example, suppose all we know about the scoring functions is that they are monotone increasing, then  Program \eqref{lp:linear-noise} can be adapted to this case simply by removing its second linearity constraint and only using the ordinal constraint   $x^\ell_j- \widetilde x^{\ell-1}_j \geq \frac{ y^\ell_j - y_j^{\ell-1}}{C} $ with a sufficiently large constant $C$.  As another example, if  reviewers are known to have  convex scoring functions,  we can change the linearity constraint to the convexity constraint $\frac{\widetilde x^\ell_j- \widetilde x^{\ell-1}_j}{y^\ell_j - y^{\ell-1}_j} \leq  \frac{\widetilde x^{\ell+1}_j- \widetilde x^\ell_j}{y^{\ell+1}_j - y^{\ell}_j} $.
These can all lead to similar guarantees as Theorem \ref{thm:linear-noise}.

%
%
%
%

Additionally, another   advantage of our optimization \LSC framework is that it can handle a \emph{mixed} set of reviewer scoring functions, \emph{if such more refined prior knowledge is available}. Specifically, in many applications, the algorithm designer may have different amount of prior knowledge about different reviewers since some (senior) reviewers have been in the system for a long time whereas some  just entered the system. Consequently, the designer may know, e.g.,  some reviewers' scoring functions are linear,  some  are convex or concave, whereas for some other reviewers, the designer   knows nothing beyond  monotonicity. Our \LSC framework can be easily adapted to handle such mixed prior knowledge since its constraints can be ``personalized'' for each reviewer. Certainly, more specific or detailed prior knowledge  about the hypothesis class will lead to better calibration, as also shown in our experiments. However, we remark that what hypothesis class is more realistic or  justifiable will likely depend on concrete applications on hand, and is beyond the scope of this paper  ---  our focus here is on the generic methodologies. We will demonstrate this kind of modeling strength of the LSC framework in the experiments section.

\section{Connection to Other Problems} \label{sec:related}
Our \LSC framework shares connections to several other problems in statistics and machine learning.
In the hope of providing further intuition and justification for our approach, in this section we elaborate on these connections.

\subsection{Linear Regression}
Linear regression~\cite{Seber12} is a classic approach to modeling the relationship between two or more variables as a linear function. \textit{Ordinary least square (OLS)} is a method to fit such a linear model by minimizing the sum of squared difference between the observed dependent variable and the predicted dependent variable. More specifically, OLS can be expressed by the formulation in Program  \eqref{lp:ols}. 
\begin{figure}[ht]
	\vspace{-5mm}
	\centering
	\begin{minipage}{.33\textwidth}
		\centering
		\begin{align}\label{lp:ols}
			\min_{\mathbf{\alpha, \beta, \epsilon}} & \quad \sum_{j=1}^M  (\epsilon^\ell)^2 \\ \nonumber
			\textup{s.t. } & \quad y^\ell = \alpha \cdot   x^\ell + \beta  + \epsilon^\ell \tag*{$ \forall  l$}
		\end{align}
	   Linear regression formulation
		\vspace{-8pt}
	\end{minipage}%
	\hspace{8mm}
	\begin{minipage}{.6\textwidth}
		\centering
		\begin{align}\label{lp:lfop}
			\min_{\mathbf{x, \alpha, \beta, \epsilon}} & \quad \sum_{j=1}^M \sum_{\ell=1}^{|I_j|} (\epsilon_j^\ell)^2 \\ \nonumber
			\textup{s.t. } & \quad y_j^\ell = \alpha_j \cdot (x(I_j^\ell) +\epsilon^\ell_j) + \beta_j \tag*{$\forall j \in [M], l \leq |I_j|$}
		\end{align}
	LSC formulation
		\vspace{-8pt}
	\end{minipage}%
\end{figure}

Observe that the formulation of linear regression  is structurally similar to our LSC \eqref{lp:framework} when $\H = \H_L$, which is formulated in Program  \eqref{lp:lfop} above. 
There are two major differences between our LSC in  Program \eqref{lp:lfop} and OLS in Program  \eqref{lp:ols}. First, the true item qualities $\mathbf{x}$ in OLS~\eqref{lp:ols}  is known, whereas in our LSC \eqref{lp:lfop}, $\mathbf{x}$ is unknown and treated as variables to optimize. Second, with the additional challenge of not knowing $\mathbf{x}$, the LSC also has an advantage --- we have multiple reviewers (indexed by $j$) with multiple linear functions/models who will generate scores for the same  $\mathbf{x}$. These make LSC relevant to  but quite different from  OLS linear regression. 

\subsection{Matrix Seriation Problem}
Our framework is also closely relevant to a well-known difficult combinatorial optimization problem called  \emph{matrix seriation}~\cite{Ha08, Liiv10}. The problem looks to find a consistent ordering of the columns of a matrix. More specifically, given a set of $n$ objects and a (partially observed) data matrix $A$ in which each row of $A$ reflects partial  observations of these objects, the goal of seriation is to find a linear ordering of these $n$ objects that best fit the observed data according to certain loss function.

The matrix seriation problem has its roots in many disciplines, with applications such as sparse matrix reordering~\cite{Atkins08}, DNA sequencing~\cite{Benzer62}, and archeological dating~\cite{Brainerd97}. There are many  variants of the matrix seriation problem. Below we   define a popular version with the $\ell_2$ loss function.

\begin{definition}[$\ell_2$-Matrix Seriation] \label{def:seriation}
Given a partially observed  matrix $A \in \RR^{m \times n}$ in which each entry $A_{ij} \in \RR$ or $A_{ij} = *$ (not observed), find matrix $B \in \RR^{m \times n}$ that minimizes 
$\norm{A-B}_2^2 = \sum_{(i,j): A_{ij} \neq *}(A_{ij} - B_{ij})^2 $
subject to $\forall i,j \in [m], q,p\in [n],$
$	B_{i,q} \leq B_{i,p} \iff B_{j,q} \leq B_{j,p}$. That is, 
in matrix $B$, the ordering of each entry according to the column is consistent across each row.

\end{definition}


Review score calibration   can   be naturally modeled as a matrix seriation problem.  We have a set of items, and a data matrix $A$ in which each row represents a reviewer's rating scores to (only)  his/her assigned items. Our goal is to find an  ordering of the items that is most consistent with the scoring matrix $A$ under the assumption that each review has a \emph{monotone} scoring function.
In the following, we show that  the seriation problem is a small variant of our \LSC problem.
\begin{theorem}\label{thm:maxtrix}
    The $\ell_2$ Matrix Seriation \eqref{def:seriation} problem can be solved by the following Functional Optimization Problem ($\H_{\mono}$ contains all monotone increasing functions): 
    \begin{align}\label{lp:seriation-0}
	\min_{\mathbf{x, f}}   & \quad \sum_{j=1}^M \sum_{\ell=1}^{|I_j|} (\epsilon_j^\ell)^2   \\ \nonumber
\textup{s.t. }  &  \quad y_j^\ell = f_j(x(I_j^\ell) ) + \epsilon_j^\ell   
\, \textit{ and } \,    f_j \in \H_{\mono}, \hspace{15mm}  \forall j \in [M], \ell \leq |I_j|
\end{align} 
\end{theorem}

%
%
%

It is worthwhile to  compare FOP~\eqref{lp:seriation-0} with our LSC~\eqref{lp:framework} with monotone scoring functions as they have very similar formulations.
The difference is that in FOP~\eqref{lp:seriation-0}, the noise term $\epsilon^\ell_j$ is added after applying the scoring function, whereas in LSC~\eqref{lp:framework} it is added to $x(I^{\ell}_j)$ inside the function. Notably,  this seemingly small difference  turns out to greatly affect the tractability of the problem. Indeed, the matrix seriation problem is known to be computationally \emph{intractable} \cite{chepoi2009seriation}. In contrast, our LSC~\eqref{lp:framework} with monotone scoring functions can be efficiently solved via convex programming  due to the noise inside scoring functions (as assumed by most previous works in this space \cite{Ge11,Roos11,Baba13,Mackay17,Wang20Debiasing}).  

\section{Experiments} \label{sec:exp}
A major challenge when evaluating the performance of calibration algorithms is that their performance measures   rely on the underlying true item qualities. Unfortunately this information is unattainable in most applications --- indeed, if we already know the true qualities, peer reviews would not be needed any more. Besides, due to anonymity of reviewers in all public peer review data that we are able to find, the underlying network of review graph cannot be recovered. Therefore, like many other works in this domain \cite{Mackay17,Wang18, Wang20Debiasing}, we evaluate our algorithms primarily based on synthetic data where we do know the true qualities. To be more realistic, the distribution parameters of our synthetic data are chosen based on ICLR 2019 review scores \cite{iclr19}. Our experiments are designed according to the review procedure of academic papers,  we thus also refer to items as \emph{papers} in this section. In addition, we include an experiment on a real-world dataset~\cite{sajjadi2016peer} in Section \ref{sec:peer-grade}. While our model still consistently demonstrates good performance, we still want to raise a caveat here that strategic manipulation could be a non-negligible factor in the process of peer grading~\cite{zarkoob2020report, stelmakh2021catch}.  Given that such strategic behavior is rarely observed in peer review setting such as  academic conferences, it therefore remains a crucial task to obtain a real-world dataset for future work of this domain.

\subsection{Evaluation on Synthesized Dataset}

\paragraph{Dataset Generation} Each paper's true quality  is drawn from  Gaussian distribution $  \N(5.32, 1.2)$   truncated  to be within $[0,10]$,\footnote{Truncation is for comparison convenience. Untruncated scores will only be easier since boundary papers are easier to distinguish. }  where the mean $5.32$ and standard deviation $1.2$ is estimated   based on the distribution of ICRL2019 review scores. Our random assignment algorithm ensures that most papers will get reviewed by roughly the same number of reviewers. Each reviewer is randomly assigned a scoring function from pre-specified family. Before applying the scoring function, a zero-mean i.i.d. Gaussian noise $\epsilon \sim \N(0, \sigma)$ is added to the true quality of each paper for each reviewer. More details about the scoring function generate as well as the assignment algorithm can be found in Appendix \ref{append:experiment}.

\paragraph{Evaluation Metrics} We adopt two different metrics to quantitively and comprehensively evaluate the performance of our models.
Let $T \subset[N]$ be the ranked list of $n$ papers with highest true qualities and $S\subset[N]$ denote the ranked list of $n$ papers selected by any algorithm.
In all our experiments, we set $|S| = |T| = 0.1 N$, i.e., an acceptance ratio of $10\%$ in our simulated conference setting.

\begin{itemize}[leftmargin=*]
	\item  \textbf{Precision} $	\rho(S, T) =  \frac{1}{|S|}\sum_{i \in S} \one[i \in T]$ measures the ratio of the top papers in $T$ that are selected into set $S$.  The higher precision is, the more papers with truly high qualities are accepted.   
	\item \textbf{Average Gap} $\sigma(S, T) =  \frac{1}{|T|}  \sum_{i \in T} x^*(i) - \frac{1}{|S|} \sum_{i \in S} x^*(i) $ measures the   gap between the average true qualities of the best papers in $T$ and average true qualities of the papers in $S$. It characterize how far the  recovered average paper quality is from the optimal. The smaller this gap is, the better.  
\end{itemize}

\paragraph{Baselines}
We consider three competitive baselines. The first is a widely used heuristics which simply uses the averaged reviewers' scores as a prediction of the paper's true quality, and accordingly select the best papers. We will refer this model as the \textbf{Average}. The second baseline is based on a \textit{quadratic program} (\textbf{QP}) proposed by \citet{Roos11}. The third baseline proposed by \citet{Ge11} uses a Bayesian model (\textbf{Bayesian}) to calibrate paper scores. The latter two baselines both assume linear reviewer scoring functions with noise that is similar to our modeling assumption. 




\paragraph{Experiment Setup} We first test the situation with  linear scoring functions in order to have fair comparisons with previous methods which mostly assume linearity. 
Parameters are set as $N = 1000, M=1000, k=5$. We compare two settings: (1) $ \sigma=0$ (\emph{noiseless} case); (2)   $\sigma=0.5$ (\emph{noisy} case).  All reported results are averaged over 20 trials.  All of our models are implemented with the Gurobi Optimizer~\cite{gurobi}\footnote{The source code can be found at \url{https://github.com/lab-sigma/lsc}}. All our algorithms can be solved very efficiently in seconds, we thus will not present running time results.

\begin{table*}[t]
	
	\caption{Performance comparison under \textbf{linear scoring} setting (LSC is our method).  
		The left and right side of the table respectively corresponds to   the \textbf{noiseless} ($\sigma=0$) and \textbf{noisy} ($\sigma=0.5$) setting. 	
	}
	\centering
		\begin{tabular}{ l  c c || c c }
			\toprule
			\diagbox{Model}{Metric} & Pre. (\%) & Avg. Gap  & Pre. (\%) & Avg. Gap  \\
			\midrule
			Average & 40.0 $\pm$ 4.0 & 0.78 $\pm$ 0.08 & 39.2 $\pm$ 4.5 & 0.80 $\pm$ 0.08 \\
			QP   &  97.1 $\pm$ 1.7 & 0.01 $\pm$ 0.01 &  69.2 $\pm$ 4.6 & 0.24 $\pm$ 0.09 \\
			Bayesian &  76.2 $\pm$ 4.0 & 0.12 $\pm$ 0.02 & 71.5 $\pm$ 3.1 & 0.17 $\pm$ 0.03 \\
			LSC (mono) & 91.7 $\pm$ 1.8 & 0.02 $\pm$ 0.01 & 75.9 $\pm$ 2.6 & 0.12 $\pm$ 0.02 \\
			LSC (linear) & \textbf{100 $\pm$ 0} & \textbf{0 $\pm$ 0} & \textbf{80.1 $\pm$ 2.9} & \textbf{0.08 $\pm$ 0.01}  \\
			\bottomrule
		\end{tabular} 
	\label{table1}
\end{table*}

\paragraph{Results}  Table~\ref{table1}  presents the experiment results that compare our methods with the three baselines.  Notably, our model consistently outperform all baselines under both metrics in both settings. In the noiseless setting with linear scoring functions, the experiment results confirm our theoretical analysis, as we indeed observe the perfect recovery by our model under this typical conference setup. In addition, while both LSC (linear), QP and Bayesian specifically models the prior knowledge of linear scoring functions, our model demonstrates the more robust performance regardless of the noisy environment in both metrics. This suggests that our model meets the objectives that we want to accept the best papers and maximize the overall qualities of the selected papers. We also investigate how the model performance changes according to the hardness of the problem instance, such as the effect of changing number of papers per reviewer $k$, the paper to reviewer ratio $N : M$, and the noise scale $\sigma$ in our empirical study and plot the results in figures. In all experiments, LSC still consistently outperforms baselines, and more details  can be found in the Appendix \ref{append:experiment}.

\subsection{Evaluations on the Peer-Grading Dataset} \label{sec:peer-grade}
\begin{table}[h]
\caption{
Performance comparisons in \textbf{Peer-Grading Dataset}.
}
\centering
\begin{tabular}{ l ccccc }
	\toprule
	\diagbox{Metric}{Model} &  Average & QP & Bayesian & LSC (mono) & LSC (linear) \\
	\midrule
	Pre. (\%) & 80.9 & 80.3 & 79.4 & 78.5 & \textbf{82.2}\\ 
	\midrule
	Avg. Gap & 0.48 &  0.40 & 0.47 & 0.55 & \textbf{0.38}\\
	\bottomrule
\end{tabular}	

\label{tab:peer-grading}
\end{table}	

To test our model beyond synthesized dataset, we use a well-known public Peer-Grading dataset from ~\cite{sajjadi2016peer} that collects the peer grading scores of questions in 6 different homework submissions in an Algorithm \& Data Structures class at the University of Hamburg. 
We use the average score graded by the TAs as the ground truth quality $x^*$ of each paper (homework submission). In Table \ref{tab:peer-grading}, we list the precision of the LSC models and other baseline on selecting the top 50\% of the first homework submission. Our model has the best calibration result. Moreover, the LSC model   shows an even more clear edge in ranking-based metrics, meaning the order of the recovered quality matches better with the ground-truth. We defer the details of these results to Appendix \ref{append:peer-grading}.

\subsection{Mixed Set of Scoring Functions and Robustness to Mis-Specified Prior Knowledge} \label{sec:robustness}
\begin{wrapfigure}[18]{r}{0.4\textwidth}
	\includegraphics[width=0.4\textwidth]{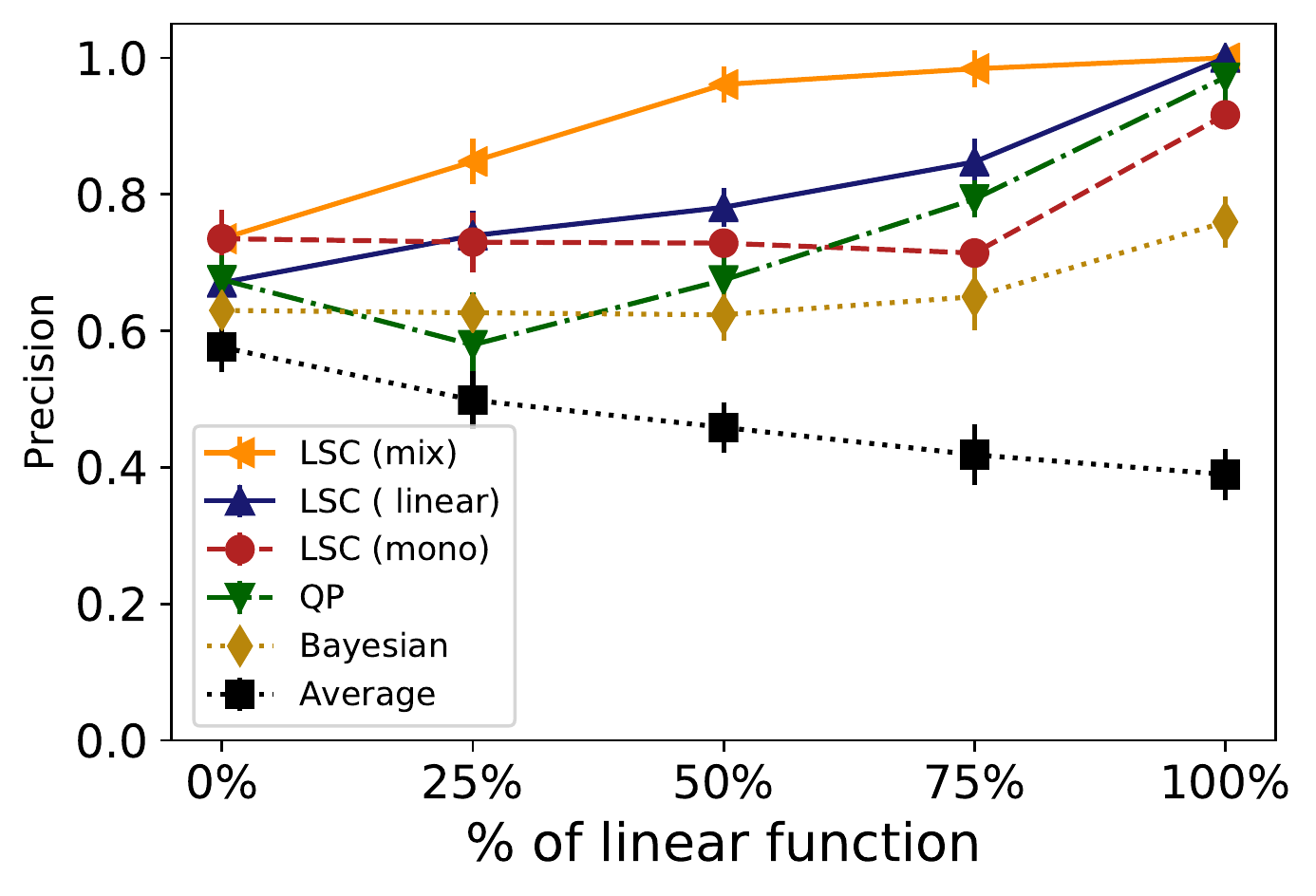}
	\caption{
			Performance comparisons in the noiseless and mixed setups with linear scoring functions and arbitrary monotone functions. Only \LSC (mix) has prior knowledge of  every reviewer's scoring function type. 
		}
	\label{fig:robustness}	
\end{wrapfigure}

To study the robustness of the different methods in scenarios where the scoring function assumption no longer holds or is mis-specified, we consider settings where some percentage of reviewers' scoring functions are instead  \emph{arbitrary random} (non-linear) monotone functions.  Figure~\ref{fig:robustness}  plots the precision curve of our methods and baselines with varying percentage of linear scoring functions. It shows that 1) our LSC (linear) consistently outperforms prior work (QP, Bayesian) that also assumes linear scoring functions. 2) As the percentage of linear function decreases, the gap between LSC (linear) and \LSC (mono) becomes closer, and when all functions are random monotone functions, LSC (linear) is $4\%$ behind LSC (mono) due to its mis-specified prior knowledge in the noiseless setting.  3) \LSC (mix) --- which has accurate prior knowledge about which reviewers are linear --- significantly outperforms \LSC (linear) and all other models, showing that good prior knowledge is indeed quite helpful for calibration. 

 Appendix \ref{append:additional-Exp} contains additional experiments for noisy scores and mixed settings with  monotone, convex  and concave functions. The results further demonstrate the usefulness of prior knowledge for calibration and thus also justified the value of having a flexible calibration methods like \LSC.
 


 \section{Conclusion}
This paper presents a simple yet powerful framework for calibration in peer review systems, which exploits both the robustness of linear regression methods and the topological structure of review graphs. Moreover, our empirical and theoretical results provide a general guideline on the assignment rules in peer review for more effective calibration.
 In follow-up work, we hope to build upon our flexible framework for the extension to a wider spectrum of scoring function hypothesis classes, as well as a broader family of peer review systems beyond the classical format of academic conference.
 For example, AAAI-2021 recently started to adopt a review process\footnote{\url{https://aaai.org/Conferences/AAAI-21/reviewing-process/}}, where in the first phase senior reviewers review all the papers and make summary rejections. The remaining papers will go to the second phase and receive additional reviews for the final decisions. It remains interesting open problems to develop novel calibration methods for such two-phase peer review process, or alternatively, to design the peer review mechanism to work better with calibration methods.

 \section{Acknowledgements}
 We thank all the anonymous reviewers for their helpful comments, especially the discussions on real-world datasets. This research is supported by a Google Faculty Research Award and the Ministry of Education, Singapore, under its Academic Research Fund Tier 2 (MOE2019-T2-1-045).

 \bibliography{main}

\begin{thebibliography}{29}
\providecommand{\natexlab}[1]{#1}
\providecommand{\url}[1]{\texttt{#1}}
\expandafter\ifx\csname urlstyle\endcsname\relax
  \providecommand{\doi}[1]{doi: #1}\else
  \providecommand{\doi}{doi: \begingroup \urlstyle{rm}\Url}\fi

\bibitem[icl(2019)]{iclr19}
\emph{7th International Conference on Learning Representations, {ICLR} 2019,
  New Orleans, LA, USA, May 6-9, 2019}, 2019. OpenReview.net.
\newblock URL \url{https://openreview.net/group?id=ICLR.cc/2019/Conference}.

\bibitem[AAA(2020)]{AAAI2020}
\emph{The Thirty-Fourth {AAAI} Conference on Artificial Intelligence, {AAAI}
  2020, The Thirty-Second Innovative Applications of Artificial Intelligence
  Conference, {IAAI} 2020, The Tenth {AAAI} Symposium on Educational Advances
  in Artificial Intelligence, {EAAI} 2020, New York, NY, USA, February 7-12,
  2020}, 2020. {AAAI} Press.
\newblock ISBN 978-1-57735-823-7.
\newblock URL \url{https://www.aaai.org/Library/AAAI/aaai20contents.php}.

\bibitem[Ammar and Shah(2012)]{Ammar12}
Ammar Ammar and Devavrat Shah.
\newblock Efficient rank aggregation using partial data.
\newblock \emph{ACM SIGMETRICS Performance Evaluation Review}, 40\penalty0
  (1):\penalty0 355--366, 2012.

\bibitem[Atkins et~al.(1998)Atkins, Boman, and Hendrickson]{Atkins08}
Jonathan~E Atkins, Erik~G Boman, and Bruce Hendrickson.
\newblock A spectral algorithm for seriation and the consecutive ones problem.
\newblock \emph{SIAM Journal on Computing}, 28\penalty0 (1):\penalty0 297--310,
  1998.

\bibitem[Baba and Kashima(2013)]{Baba13}
Yukino Baba and Hisashi Kashima.
\newblock Statistical quality estimation for general crowdsourcing tasks.
\newblock In \emph{Proceedings of the 19th ACM SIGKDD international conference
  on Knowledge discovery and data mining}, pages 554--562, 2013.

\bibitem[Benzer(1962)]{Benzer62}
Seymour Benzer.
\newblock The fine structure of the gene.
\newblock \emph{Scientific American}, 206\penalty0 (1):\penalty0 70--87, 1962.

\bibitem[Brainerd(1997)]{Brainerd97}
George~W Brainerd.
\newblock The place of chronological ordering in archaeological analysis.
\newblock In \emph{Americanist Culture History}, pages 301--313. Springer,
  1997.

\bibitem[Chepoi et~al.(2009)Chepoi, Fichet, and Seston]{chepoi2009seriation}
Victor Chepoi, Bernard Fichet, and Morgan Seston.
\newblock Seriation in the presence of errors: Np-hardness of l-infinity
  fitting robinson structures to dissimilarity matrices.
\newblock \emph{Journal of classification}, 26\penalty0 (3):\penalty0 279--296,
  2009.

\bibitem[Freund et~al.(2003)Freund, Iyer, Schapire, and Singer]{Freund03}
Yoav Freund, Raj Iyer, Robert~E Schapire, and Yoram Singer.
\newblock An efficient boosting algorithm for combining preferences.
\newblock \emph{Journal of machine learning research}, 4\penalty0
  (Nov):\penalty0 933--969, 2003.

\bibitem[Ge et~al.()Ge, Welling, and Ghahramani]{Ge11}
Hong Ge, Max Welling, and Zoubin Ghahramani.
\newblock A bayesian model for calibrating reviewer scores.

\bibitem[Gurobi~Optimization(2021)]{gurobi}
LLC Gurobi~Optimization.
\newblock Gurobi optimizer reference manual, 2021.
\newblock URL \url{http://www.gurobi.com}.

\bibitem[Hahsler et~al.(2008)Hahsler, Hornik, and Buchta]{Ha08}
Michael Hahsler, Kurt Hornik, and Christian Buchta.
\newblock Getting things in order: an introduction to the r package seriation.
\newblock \emph{Journal of Statistical Software}, 25\penalty0 (3):\penalty0
  1--34, 2008.

\bibitem[Harzing et~al.(2009)Harzing, Baldueza, Barner-Rasmussen, Barzantny,
  Canabal, Davila, Espejo, Ferreira, Giroud, Koester, et~al.]{Harzing09}
Anne-Wil Harzing, Joyce Baldueza, Wilhelm Barner-Rasmussen, Cordula Barzantny,
  Anne Canabal, Anabella Davila, Alvaro Espejo, Rita Ferreira, Axele Giroud,
  Kathrin Koester, et~al.
\newblock Rating versus ranking: What is the best way to reduce response and
  language bias in cross-national research?
\newblock \emph{International Business Review}, 18\penalty0 (4):\penalty0
  417--432, 2009.

\bibitem[Kiayias et~al.(2020)Kiayias, Kohlweiss, Wallden, and
  Zikas]{Kiayias2020}
Aggelos Kiayias, Markulf Kohlweiss, Petros Wallden, and Vassilis Zikas.
\newblock \emph{Public-Key Cryptography-PKC 2020}.
\newblock Springer, 2020.

\bibitem[Liiv(2010)]{Liiv10}
Innar Liiv.
\newblock Seriation and matrix reordering methods: An historical overview.
\newblock \emph{Statistical Analysis and Data Mining: The ASA Data Science
  Journal}, 3\penalty0 (2):\penalty0 70--91, 2010.

\bibitem[MacKay et~al.(2017)MacKay, Kenna, Low, and Parker]{Mackay17}
Robert~S MacKay, Ralph Kenna, Robert~J Low, and Sarah Parker.
\newblock Calibration with confidence: a principled method for panel
  assessment.
\newblock \emph{Royal Society open science}, 4\penalty0 (2):\penalty0 160760,
  2017.

\bibitem[Mitliagkas et~al.(2011)Mitliagkas, Gopalan, Caramanis, and
  Vishwanath]{Mitliagkas11}
Ioannis Mitliagkas, Aditya Gopalan, Constantine Caramanis, and Sriram
  Vishwanath.
\newblock User rankings from comparisons: Learning permutations in high
  dimensions.
\newblock In \emph{2011 49th Annual Allerton Conference on Communication,
  Control, and Computing (Allerton)}, pages 1143--1150. IEEE, 2011.

\bibitem[Negahban et~al.(2012)Negahban, Oh, and Shah]{Negahban12}
Sahand Negahban, Sewoong Oh, and Devavrat Shah.
\newblock Iterative ranking from pair-wise comparisons.
\newblock \emph{Advances in neural information processing systems},
  25:\penalty0 2474--2482, 2012.

\bibitem[Paul(1981)]{Paul81}
SR~Paul.
\newblock Bayesian methods for calibration of examiners.
\newblock \emph{British Journal of Mathematical and Statistical Psychology},
  34\penalty0 (2):\penalty0 213--223, 1981.

\bibitem[Rokeach(1968)]{Rokeach68}
Milton Rokeach.
\newblock The role of values in public opinion research.
\newblock \emph{Public Opinion Quarterly}, 32\penalty0 (4):\penalty0 547--559,
  1968.

\bibitem[Roos et~al.(2011)Roos, Rothe, and Scheuermann]{Roos11}
Magnus Roos, J{\"o}rg Rothe, and Bj{\"o}rn Scheuermann.
\newblock How to calibrate the scores of biased reviewers by quadratic
  programming.
\newblock In \emph{Proceedings of the AAAI Conference on Artificial
  Intelligence}, volume~25, 2011.

\bibitem[Sajjadi et~al.(2016)Sajjadi, Alamgir, and von
  Luxburg]{sajjadi2016peer}
Mehdi~SM Sajjadi, Morteza Alamgir, and Ulrike von Luxburg.
\newblock Peer grading in a course on algorithms and data structures: Machine
  learning algorithms do not improve over simple baselines.
\newblock In \emph{Proceedings of the third (2016) ACM conference on Learning@
  Scale}, pages 369--378, 2016.

\bibitem[Seber and Lee(2012)]{Seber12}
George~AF Seber and Alan~J Lee.
\newblock \emph{Linear regression analysis}, volume 329.
\newblock John Wiley \& Sons, 2012.

\bibitem[Stelmakh et~al.(2021)Stelmakh, Shah, and Singh]{stelmakh2021catch}
Ivan Stelmakh, Nihar~B Shah, and Aarti Singh.
\newblock Catch me if i can: Detecting strategic behaviour in peer assessment.
\newblock In \emph{Proceedings of the AAAI Conference on Artificial
  Intelligence}, volume~35, pages 4794--4802, 2021.

\bibitem[Wallach et~al.(2019)Wallach, Larochelle, Beygelzimer,
  d'Alch{\'{e}}{-}Buc, Fox, and Garnett]{nips2019}
Hanna~M. Wallach, Hugo Larochelle, Alina Beygelzimer, Florence
  d'Alch{\'{e}}{-}Buc, Emily~B. Fox, and Roman Garnett, editors.
\newblock \emph{Advances in Neural Information Processing Systems 32: Annual
  Conference on Neural Information Processing Systems 2019, NeurIPS 2019,
  December 8-14, 2019, Vancouver, BC, Canada}, 2019.
\newblock URL \url{https://proceedings.neurips.cc/paper/2019}.

\bibitem[Wang and Shah(2019)]{Wang18}
Jingyan Wang and Nihar~B Shah.
\newblock Your 2 is my 1, your 3 is my 9: Handling arbitrary miscalibrations in
  ratings.
\newblock In \emph{Proceedings of the 18th International Conference on
  Autonomous Agents and MultiAgent Systems}, pages 864--872, 2019.

\bibitem[Wang et~al.(2020)Wang, Stelmakh, Wei, and Shah]{Wang20Debiasing}
Jingyan Wang, Ivan Stelmakh, Yuting Wei, and Nihar~B Shah.
\newblock Debiasing evaluations that are biased by evaluations.
\newblock \emph{arXiv preprint arXiv:2012.00714}, 2020.

\bibitem[Wang et~al.(2013)Wang, Wang, Li, He, and Liu]{wang2013theoretical}
Yining Wang, Liwei Wang, Yuanzhi Li, Di~He, and Tie-Yan Liu.
\newblock A theoretical analysis of ndcg type ranking measures.
\newblock In \emph{Conference on Learning Theory}, pages 25--54. PMLR, 2013.

\bibitem[Zarkoob et~al.(2020)Zarkoob, Fu, and Leyton-Brown]{zarkoob2020report}
Hedayat Zarkoob, Hu~Fu, and Kevin Leyton-Brown.
\newblock Report-sensitive spot-checking in peer-grading systems.
\newblock In \emph{Proceedings of the 19th International Conference on
  Autonomous Agents and MultiAgent Systems}, pages 1593--1601, 2020.

\end{thebibliography}
 \bibliographystyle{plainnat}

\newpage
\section*{Checklist}


\begin{enumerate}

\item For all authors...
\begin{enumerate}
  \item Do the main claims made in the abstract and introduction accurately reflect the paper's contributions and scope?
    \answerYes{} See Section \ref{sec:intro}
  \item Did you describe the limitations of your work? 
    \answerYes{} See Section \ref{sec:intro}
  \item Did you discuss any potential negative societal impacts of your work?
    \answerNo{} We are not aware of any negative social impact of our work.
  \item Have you read the ethics review guidelines and ensured that your paper conforms to them?
    \answerYes{}
\end{enumerate}

\item If you are including theoretical results...
\begin{enumerate}
  \item Did you state the full set of assumptions of all theoretical results?
    \answerYes{} See Section \ref{sec:model}
	\item Did you include complete proofs of all theoretical results?
    \answerYes{} See Appendix \ref{append:proof}
\end{enumerate}

\item If you ran experiments...
\begin{enumerate}
  \item Did you include the code, data, and instructions needed to reproduce the main experimental results (either in the supplemental material or as a URL)?
    \answerYes{} See supplemental material
  \item Did you specify all the training details (e.g., data splits, hyperparameters, how they were chosen)?
    \answerYes{} See Section \ref{sec:exp} and Appendix \ref{append:experiment}
	\item Did you report error bars (e.g., with respect to the random seed after running experiments multiple times)? 
    \answerYes{} See Section \ref{sec:exp} and Appendix \ref{append:experiment}
	\item Did you include the total amount of compute and the type of resources used (e.g., type of GPUs, internal cluster, or cloud provider)? 
    \answerNo{} Our algorithm can be executed on a laptop CPU.
\end{enumerate}

\item If you are using existing assets (e.g., code, data, models) or curating/releasing new assets...
\begin{enumerate}
  \item If your work uses existing assets, did you cite the creators?
    \answerNA{}
  \item Did you mention the license of the assets?
    \answerNA{}
  \item Did you include any new assets either in the supplemental material or as a URL?
    \answerNA{}
  \item Did you discuss whether and how consent was obtained from people whose data you're using/curating?
    \answerNA{}
  \item Did you discuss whether the data you are using/curating contains personally identifiable information or offensive content?
    \answerNA{}
\end{enumerate}

\item If you used crowdsourcing or conducted research with human subjects...
\begin{enumerate}
  \item Did you include the full text of instructions given to participants and screenshots, if applicable?
    \answerNA{}
  \item Did you describe any potential participant risks, with links to Institutional Review Board (IRB) approvals, if applicable?
    \answerNA{}
  \item Did you include the estimated hourly wage paid to participants and the total amount spent on participant compensation?
    \answerNA{}
\end{enumerate}

\end{enumerate}

\newpage

\newpage
\appendix

\title{Supplemental Material for \\ Least Square Calibration for Peer Reviews}

\section{Missing Proofs}\label{append:proof}

\subsection{Proof of Theorem \ref{thm:perfect}}\label{append:perfect-proof}
%
\begin{theorem*}[Restatement of Theorem \ref{thm:perfect} ] 
A review graph $G$ is   recovery-resilient \emph{if and only if }   the review graph $G$ has a doubly-connected component $S$ that covers all items, i.e. $C(S) = [N]$.  	
\end{theorem*}

\begin{proof} 
\textbf{Proof of Sufficiency. } 
We first consider the ``if'' direction.  That is, the qualities of all items covered by a doubly-connected component can always be perfectly recovered by any solution of LP \eqref{lp:linear}. Without loss of generality, in the following argument, we fix any given review scores, any solution of LP. Following the design of Algorithm \ref{alg:union}, we present a bottom-up induction proof. 

\textbf{Base case:} the smallest unit of a doubly-connected component is a single vertex, which represents some reviewer $i$. According to the linear constraint in LP \eqref{lp:linear}, the recovered qualities $x$ of all items in $I_i$ must be a linear transformation of their review scores $y^i$ by reviewer $i$. As these review scores $y^i$ are also a linear transformation of the true qualities $x^*$, so the recovered qualities $x$ solved by the LP \eqref{lp:linear} must be a linear transformation of the true qualities $x^*$. In another word, the solution of the LP \eqref{lp:linear} must perfectly recover the qualities of all items in $I_i$.

\textbf{Inductive case:} Given two doubly-connected components $A, B$, and the solution of LP \eqref{lp:linear} restricted to $A,B$  is a perfect recovery respectively to the items covered by $A$ and $B$.  This means, by definition, we have $k_A,k_B >0$  and $b_A, b_2$ such that $x_A = k_A x^*_A+b_A, x_B = k_B x^*_B+b_B$,  where $x_A, x_A^*$ and $x_B, x_B^* $ are the recovered \emph{vector} of qualities and true qualities indexed by  items in $A, B$. 
	According to Algorithm \ref{alg:union}, if $A$ and $B$ share at least $2$ commonly covered items, then their union $D=A\union B$ is also a doubly-connected component. We show that the solution of \eqref{lp:linear} must perfectly recover the qualities of items covered by $D$.

Let $u, v$ be two of the items covered by both component $A$ and $B$. Denote their recovered (by LP \eqref{lp:linear} restricted to $D$) and true qualities be $x_u, x_u^*$ and $x_v, x_v^*$ respectively. Let $x_A, x_B$ denote the recovered $x$ when restricted to $A,B$ respectively. By definition, $x_A$ [resp. $x_B$] is a feasible solution to LP \eqref{lp:linear} restricted to $A$ [resp. $B$]. Our induction hypothesis thus implies that there are the two linear functions $x_A = k_A x^*_A+b_A, x_B = k_B x^*_B+b_B$. Note that $x_u, x_v$ should appear in both vector $x_A$ and $x_B $.  That means these two linear functions $x_A = k_A x^*_A+b_A, x_B = k_B x^*_B+b_B$ intersect at both $(x_u, x_u^*)$ and $(x_v, x_v^*)$. Since item true qualities $x^*$ are assume to be unequal for any two   different items, the only possibility when the two linear functions   have two intersections is that  they are identical, i.e., $k_A=k_B, b_A=b_B$. This implies the entire $x_D$  vector where $D = A\cup B$ must satisfy $x_D = k_A \cdot x^*_D  + b_A$. That is, recovered qualities in both $x_A, x_B$ follow the same linear transformation from the true qualities. Therefore, all items covered by $D$ can be perfectly recovered by the solution of LP \eqref{lp:linear}.

%

\textbf{Proof of   Necessity. }   
For the ``only if'' direction, we prove its contrapositive statement. That is, if  there is no doubly-connected component $S$ that covers all items in the review graph, then  there must exist some paper assignment that induces review graph $G$ and some review scores under which the solution to LP  \eqref{lp:linear} cannot perfectly recover the true scores.  

We start with a few simplifications, that are without loss of generality. First, according to the above proof of  ``if'' direction, we  know that     papers within any doubly connected component in the review graph can be perfectly recovered. This means we can replace each of these components by a single vertex, i.e., a single reviewer,  with some linear scoring function, who reviewed all the items covered by this doubly connected component. After this transformation, the new review graph will have at most one edge connecting any two vertices. Second, to construct paper assignment that induces this review graph, we will let  each edge $e=(i,j)$ in the given review graph correspond to a unique item $e$, which is only reviewed by reviewer $i, j$, but no one else.  

Let $x \in \RR^N$ be an arbitrary vector solution to LP \eqref{lp:linear} which contains the recovered qualities of all the $N$ items. Without loss of generality, we will pick $x$ as the true quality $x^*$ since we know $x^*$ must be a feasible solution as well.  Next we will show   there exists another solution $\tilde{x}\in \RR^N$ to the LP \eqref{lp:linear} such that $\tilde{x}$ is not linear to $x^*$ (i.e., their corresponding entries do not have linear relation). This implies perfect recoverability is not possible  by definition.  

To construct $\tilde{x}$, we will craft a linear function for every reviewer $i$ determined by coefficient $k_i$ and constant $b_i$.  That is, for every item $u$ reviewed by reviewer $i$, we let $\tilde{x}_u = k_i x^*_u + b_i$, where   $k_i, b_i$ are to be determined later.  
We wish to set item $u$'s constructed score as $\tilde{x}_u$. If we could succeed in doing so, then as long as  $k_i \not = k_j$ for all $i,j$, then $\tilde{x}  $  cannot have a linear relation with $x^*$, which completes our proof.    However,  there are some constraints to be satisfied when  setting $\tilde{x}_u$ as item $u$'s constructed score, which is why we have to pick $\{k_i, b_i\}_{i\in [M]}$ and $x^*$  carefully. 
The constraints come  from   edges of the graph: each edge $e=(i,j)$ connecting reviewer $i$ and $j$ corresponds to an item $e$ that reviewer $i,j$ both reviewed. This will require the constructed scores, when viewed  from $i$'s and $j$'s perspective, have to be consistent, i.e.,    $$ k_i x^*_{e} + b_i = k_j x^*_{e} + b_j , \quad  \forall e $$ which both equal the constructed $\tilde{x}_{e}$.   

Since   recovery-resilience of a review graph needs to hold for any given underlying true qualities, to disapprove it  we only need to identify one set of true item qualities to satisfy our construction. Towards that end, we will use the following construction: $k_i=\sqrt{p_i}, b_i=p_i, \forall i\in [M]$, where $p_i$ is the $i$th smallest prime number from $2$. Given these $\{k_i, b_i\}_{i\in [M]}$ , we then let  
$$x^*_{e} = -\frac{b_i-b_j}{k_i-k_j} =- \frac{p_i-p_j}{\sqrt{p_i}-\sqrt{p_j}}=- \sqrt{p_i}- \sqrt{p_j} , \qquad   \text{ for all edge }e=(i,j). $$

It is easy to verify that the above construction does satisfy $ k_i x^*_{e} + b_i = k_j x^*_{e} + b_j$ for any edge $e = (i,j)$. Moreover, no two edges with $e=(i,j), e'=(i',j')$ can have their quality $\tilde{x}_e = \tilde{x}_{e'}$  because for any four $p_i,p_j,p_{i'},p_{j'}$ with at least two unique prime $p_i,p_{i'}$, it is impossible that $\sqrt{p_i}+\sqrt{p_j} = \sqrt{p_{i'}}+\sqrt{p_{j'}}$. To see this,  if we take the square of both side, this will lead to $2\sqrt{p_ip_j}-2\sqrt{p_{i'}p_{j'}} = p_i+p_j + p_{i'} + p_{j'} $. Now if we take the square of both sides again, we have $-8\sqrt{p_ip_jp_{i'}p_{j'}} = (p_i+p_j + p_{i'} + p_{j'})^2- 4(p_ip_j+p_{i'}p_{j'})$.  However, the RHS is rational, yet the LHS must be irrational since at least $p_i, p_{i'}$ are unique prime number, a contradiction.  Therefore, our construction of $\tilde{x}$ is indeed valid. 
This  concludes that the review graph with no more than one edge cannot be  recovery-resilient.

\end{proof}

\subsection{Proof of Theorem \ref{thm:linear-noise}}\label{append:proof-linear-noise}
\begin{theorem*}[Restatement]
    Convex Program \eqref{lp:linear-noise} is equivalent to \LSC with $\H= \H_{L}(C)$ in the following sense:
    for any optimal solution $(\mathbf{x}^*, \mathbf{\epsilon}^*)$ to \eqref{lp:linear-noise},
    there exists $\mathbf f^* = \{f^*_j \in \H_{L}(C)\}_{j \in [M]}$ such that $(\mathbf{x}^*, \mathbf{\epsilon}^*, \mathbf{f}^*)$ is optimal to \eqref{lp:framework}.
\end{theorem*}

\begin{proof}
	For any optimal solution $(\mathbf{x}^*, \mathbf{\epsilon}^*)$ to \eqref{lp:linear-noise}, we can linearly interpolate points $\{ ( x^*(I^\ell_j) + \epsilon^{\ell*}_j, y^\ell_j)\}_{l}$ and construct the linear function
   $f^*_j(x) = \alpha x + \beta$ with $\alpha = \frac{y^\ell_j - y^{\ell-1}_j}{\widetilde x^\ell_j- \widetilde x^{\ell-1}_j}$ and $\beta = y^\ell_j - \alpha x^*(I^\ell_j)$, for all $j \in [M], 2 \leq \ell \leq |I_j|$.
	Therefore, $(\mathbf{x}^*, \mathbf{\epsilon}^*, \mathbf{f}^* = \{f^*_j\}_{j \in [M]})$ is a \emph{feasible solution} to \LSC and thus the optimal objective of \LSC is at least that of LP \eqref{lp:linear-noise}.

	To show that the optimal objective of LCS is at most that of LP \eqref{lp:linear-noise}, observe that any feasible solution to \LSC must satisfies the \textit{linear equality constraint} of LP \eqref{lp:linear-noise} because
	\begin{align*}
		\frac{1}{\alpha_j} = \frac{\widetilde x^\ell_j - \widetilde x^{\ell-1}_j}{y^\ell_j - y^{\ell-1}_j} = \frac{\widetilde x^{\ell+1}_j - \widetilde x^\ell_j}{y^{\ell+1}_j - y^{\ell}_j}  \quad \forall j \in [M], 2 \leq \ell \leq |I_j|
	\end{align*}
	which is precisely the \textit{linear equaltiy constraints} of LP \eqref{lp:linear-noise}. This implies that the feasible region of LP \eqref{lp:linear-noise} contains the feasible region of \LSC restricted to $\mathbf{x}, \mathbf{\epsilon}$. Therefore, the optimal objective of \LSC is also at most that of LP \eqref{lp:linear-noise}, as desired.
\end{proof}

\subsection{Proof of Theorem \ref{thm:maxtrix}}\label{append:maxtrix}

\begin{theorem*}[Restatement]
    The $\ell_2$ Matrix Seriation \eqref{def:seriation} problem can be solved by the following Functional Optimization Problem.
    \begin{align}\label{lp:seriation}
     \min_{\mathbf{x, f}} & \quad \sum_{j=1}^M \sum_{\ell=1}^{|I_j|} (\epsilon_j^\ell)^2 \\ \nonumber
     \textup{s.t. } & \quad y_j^\ell = f_j(x(I_j^\ell) ) + \epsilon_j^\ell \tag*{$\forall j \in [M], \ell \leq |I_j|$} \\ \nonumber
     & \quad f_j \in \H_{\mono} \tag*{$j\in[M]$}
    \end{align}
\end{theorem*}

\begin{proof}
We can represent reviewers' scores as a matrix $A \in \RR^{m\times n}$ where each row represents the scores given by a specific reviewer, and each column represents the scores received by a specific item. Therefore, $A_{i,j}$ represents reviewer $i$'s score for item $j$. Note that $A_{i,j}$ is a partial matrix. $A_{i,j}$ is empty if reviewer $i$ does not review item $j$.
Starting with an empty matrix $B$, we fill $B_{j, I_j^\ell}$ with $f_j(x(I^\ell_j)) = y_j^\ell - \epsilon_j^\ell$ for all $j \in [M], \ell \leq |I_j|$, which is obtained from the solution of the Functional Optimization Problem \eqref{lp:seriation}.

Note that
\begin{align*}
	\norm{A-B}_2 = \sum_{j=1}^M \sum_{\ell=1}^{|I_j|} (y_j^\ell - f_j(x(I_j^\ell)))^2 = \sum_{j=1}^M \sum_{\ell \leq |I_j|} (\epsilon_j^\ell)^2
\end{align*}
which is exactly what FOP~\eqref{lp:seriation} minimizes.

In addition, because all scoring functions are monotonically increasing, it means each row of $B$ preserves the order according to items' qualities $x$. In other words, for any $i, j \in [m], p, q \in [n]$ we have
\begin{align*}
   B_{i,q} \leq B_{i,p} \iff B_{j,q} \leq B_{j,p}
\end{align*}

Therefore, \eqref{lp:seriation} solves the matrix seriation problem.
\end{proof}

\section{Additional Details and Results for Experiments}\label{append:experiment}
In this section, we provide a detailed description of our data generation procedure, as well as various experiments to better understand the strength and limitation of our calibration framework. 

\subsection{Dataset Generation} Our synthesized data generation follows the procedure below. For the ease of reproduction, we also include the implementation details in the our released code in supplemental materials.

\paragraph{Paper qualities.} Without loss of generality, we restrict the true qualities of papers to be within $[0, 10]$. We randomly sample the true quality of each paper from a Gaussian distribution $x \sim \N(5, 1.6)$. This choice of parameter ensures roughly $99.8\%$ of the true qualities fall within $[0.2, 9.8]$. We truncate the value to $0$ or $10$ if any sample falls below or above the $[0, 10]$ interval.

\paragraph{Paper assignment and review graph.} We randomly assign each paper to a pool of reviewers under the natural constraint that a paper should be reviewed at least $\lfloor kM/N \rfloor$, according to the Algorithm \ref{alg:assign-random}. 
Meanwhile, for the experiment shown in Figure \ref{fig:doubly-robustness}, we generate the review graph of double connectivity, according to Algorithm \ref{alg:assign-doubly}. The function $\texttt{choose}(S, k)$ there is to randomly sample from the set $S$ for $k$ elements without replacement. 

\begin{algorithm}
	\caption{\texttt{Random Assignment}}
	\label{alg:assign-random}
	\begin{algorithmic}[1]
		\INPUT $N$ papers, $M$ reviewer, $k$ papers per reviewer 
		\STATE $b \gets \lfloor kM/N \rfloor $ \COMMENT{minimum number of reviews requires for each paper}
		\STATE $V \gets \emptyset$ \COMMENT{set of papers that have met the minimum requirement}
		\STATE $\texttt{usage} \gets \{\} $  \COMMENT{map for the number of reviews of each paper}
		\FOR {$i \in \{1 \dots M\}$}
		\STATE $S_1 \gets \texttt{choose}\left( [N] / V, \min(k, N-|V|)\right)$ \\
		\COMMENT{sample as many papers that have not met the minimum requirement}
		\STATE $S_2 \gets \texttt{choose}\left(V, k-|S_1|)\right)$ \\
		\COMMENT{sample the remaining papers so the reviewer gets $k$ papers}		
		\STATE $T_i \gets S_1 \cup S_2$
		\FOR {$j \in I_i$}
		\STATE $\texttt{usage}[j] \gets \texttt{usage}[j]+ 1 $
		\IF {$\texttt{usage}[j] >= b $}
		\STATE $V \gets V + j $
		\ENDIF
		\ENDFOR
		\ENDFOR
		\STATE \textbf{return} $\{ I_i\}_{i\in [M]}$
	\end{algorithmic}
\end{algorithm}

\begin{algorithm}
	\caption{\texttt{Random Assignment with Double Connectivity}}
	\label{alg:assign-doubly}
	\begin{algorithmic}[1]
		\INPUT $N$ papers, $M$ reviewer, $k > 2$ papers per reviewer 
		\STATE $b \gets \lfloor kM/N \rfloor $  \COMMENT{minimum number of reviews requires for each paper}
		\STATE $\texttt{usage} \gets \{\} $  \COMMENT{map for the number of reviews of each paper}
		\STATE $ V \gets \emptyset $ \COMMENT{set of papers that have met the minimum requirement}
		\STATE $ I_1 \gets \texttt{choose}([N], k) $ \COMMENT{sample any k papers for reviewer 1}
		\STATE $ T \gets I_1 $ \COMMENT{set of papers that have been assigned}
		\STATE $ \texttt{usage}[j] \gets 1,\quad \forall j \in I_1 $
		\FOR {$i \in \{2 \dots M\}$}
		\STATE $S_1 \gets \texttt{choose}([N]/T, \max(k-2, N-|T|))$ \\ 
		\COMMENT{sample as many unassigned papers}
		\STATE $S_2 \gets \texttt{choose}(T/V, \min(2, |T| - |V|))$ \\
		 \COMMENT{sample as many as 2 assigned papers that have not met minimum requirement }
		\STATE $S_3 \gets \texttt{choose}(V, k-|S_1|-|S_2|)  $ \\
		\COMMENT{sample the remaining papers so the reviewer gets $k$ papers}		
		\STATE $I_i \gets S_1 \cup S_2 \cup S_3$
		\STATE $T \gets T \cup I_i$
		\FOR {$j \in I_i$}
		\STATE $\texttt{usage}[j] \gets \texttt{usage}[j]+ 1 $
		\IF {$\texttt{usage}[j] >= b $}
		\STATE $V \gets V + j $
		\ENDIF
		\ENDFOR
		\ENDFOR
		\STATE \textbf{return} $\{ I_i\}_{i\in [M]}$
	\end{algorithmic}
\end{algorithm}

\paragraph{Paper scores.} We randomly assign a scoring function to each reviewer to compute their review scores.

	To generate a linear scoring function, $f(x) = kx + b$, we draw the parameter $k \sim \mathcal{U}(0, 2)$ and $b \sim \N(0,2)$.

	To generate a concave function, we take a random linear combination between a set of monotone concave functions $\{c_2 \cdot \sqrt{x}, c_3 \cdot \sqrt[3]{x}, c_4 \cdot \sqrt[4]{x}\}$, where the weight of each function is sampled uniformly random according to $c_p \sim \mathcal{U}(1, \frac{10}{\sqrt[p]{10}}), \forall p\in \{1,2,3,4\}$.

	To generate a convex function, we take a random linear combination between a set of monotone convex functions $\{c_1 \cdot x^{2}, c_2 \cdot x^{2.5}, c_3 \cdot x^3\}$ where the weight of each function is sampled uniformly random according to $c_1, c_2, c_3 \sim \mathcal{U}(0,1)$.

	To generate an arbitrary monotone function, we randomly sample $k$ values from $[0,10]$ in increasing order and assign them to papers in the corresponding order.

	In addition, for the noisy case, a zero-mean Gaussian error $\epsilon \sim \N(0, \sigma)$ is added to the true qualities of the papers for each reviewer-paper pair as the percetion error before applying each reviewer's scoring function.

\subsection{Experiments for Remark \ref{rmk:doubly}} \label{append:remark3.5}
\begin{figure}[ht]
	\centering
	\includegraphics[width=0.5\textwidth]{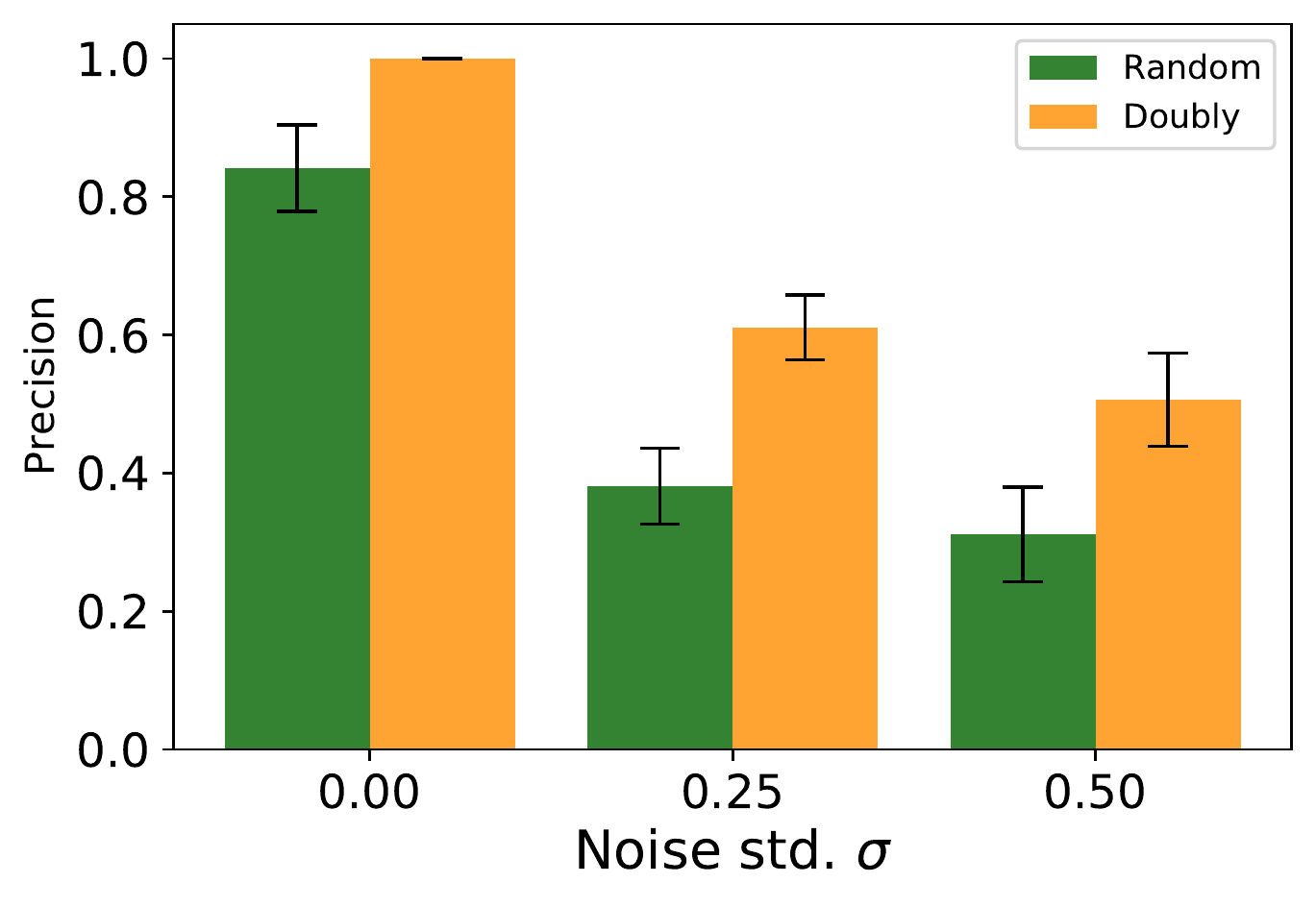}
	\caption{
			Performance comparisons of LSC with linear constraints between \textbf{randomly generated review graph} (green bar) and \textbf{randomly generated doubly connected reviewed graph} (orange bar)  under different level of noise scale with $N=1000, M=350, k=5$. 
		}
	\label{fig:doubly-robustness}
\end{figure}
In Figure \ref{fig:doubly-robustness}, we are able to directly verify our Theorem \ref{thm:perfect} that a randomly generated graph of double connectivity is indeed recovery-resilient, as it is able to achieve the bona fide perfect precision in the noiseless setting. In addition, the topological structure of double connectivity is also less prone to the perception noise, compared to a review graph generated from a completely random assignment.

\subsection{Additional Experiments and Discussions on the Effects of Prior Knowledge} \label{append:prior}
\begin{figure}[h]
	\centering
	\includegraphics[width=0.5\linewidth]{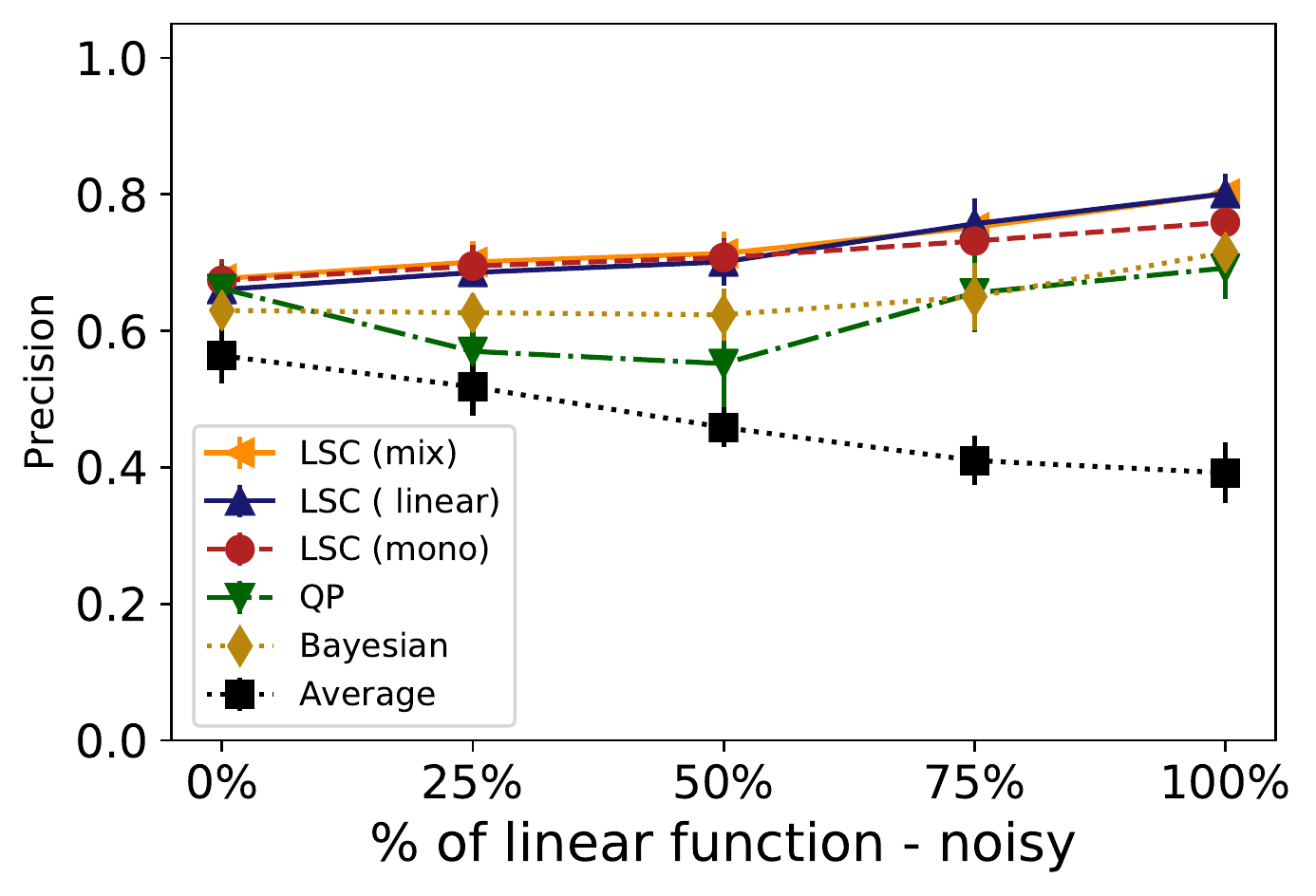}
	\caption{
			Performance comparisons in the noisy ($\sigma=0.5$) and mixed setups with linear scoring functions and arbitrary monotone  functions. Only LSC (mix) has prior knowledge of every reviewer's scoring function type. 
		}
	\label{fig:robustness-noisy}
\end{figure}

In Section \ref{sec:robustness}, we demonstrate the robustness of our model under the situation when the prior knowledge is misspecified. In Figure \ref{fig:robustness-noisy}, we can see that, under the noisy setting, our model still outperforms all the baselines. However, it turns out that the influence of whether the prior knowledge is misspecified or not is rather minimal, compared the noiseless setting. Similar pattern is observed in Table \ref{tab:mixed}, where the scoring functions consist of $1/3$ monotonic increasing function, $1/3$ convex functions, and $1/3$ concave functions (all randomly generated). Only LSC (mix) has the exact prior knowledge of  every reviewer's scoring function type. The LSC model can effectively utilize such the prior knowledge. We can see significant improvement of the precision under noiseless setting, though such improvement is minor in the more noisy setting. On one hand, our model is indeed able to fit for a wide range of function classes so that more prior knowledge can lead to better calibration. On the other hand, these observations also suggest that with the presence of more perception noise, it could be more beneficial to pick a robust model rather than fixating on an accurate prior knowledge for better calibration results. This justifies the usage of our model for the cases even without strong prior knowledge. 

\begin{table*}
	\begin{center}
	\begin{tabular}{ |c | c c || c c|}
	\hline
	\diagbox{Model}{Metric} & Pre. (\%) & Avg. Gap & Pre. (\%) & Avg. Gap   \\
	\hline
	Average & 39.9 $\pm$ 2.5 & 0.70 $\pm$ 0.06  & 38.6 $\pm$ 3.5       & 0.76 $\pm$ 0.07   \\
	QP  & 76.4 $\pm$ 3.0 & 0.13 $\pm$ 0.02  & 69.8 $\pm$ 4.5   & 0.21 $\pm$ 0.05    \\
	Bayesian   &  54.4 $\pm$ 3.1 & 0.40 $\pm$ 0.03  &  51.1 $\pm$ 3.6 & 0.48 $\pm$ 0.04 \\
	LSC (mono) & 76.4 $\pm$ 3.9 & 0.13 $\pm$ 0.03 &  68.8 $\pm$ 3.2	& 0.22  $\pm$ 0.02	 \\
	LSC (linear) & 76.4 $\pm$ 3.8 & 0.14 $\pm$ 0.02	 & 67.7 $\pm$ 3.1  & 0.22 $\pm$ 0.03	\\
	LSC (mixed) & \textbf{93.2 $\pm$ 2.5} & \textbf{0.02 $\pm$ 0.01} & \textbf{71.1 $\pm$ 3.4} & \textbf{0.18 $\pm$ 0.03}  \\
	\hline
	\end{tabular} \\
	\end{center}
	\caption{Experimental results for \textbf{mixed scoring functions} setting on the baselines and LSC with monotone,  linear or mixed (given the prior knowledge) constraints. Each entry contains the mean and standard deviation   over 20 trials. The table on the left and right side respectively shows the results in the \textbf{noiseless} ($\sigma=0$) and \textbf{noisy} ($\sigma=0.5$) setting. }
	\label{tab:mixed}
\end{table*}

\subsection{Additional Empirical Study}\label{append:additional-Exp}
In real applications such academic conferences, the calibration framework could face very different setups. For example, the PKC-2020~\cite{Kiayias2020} conference has 180 submission with an acceptance rate of 0.24 while the AAAI-2020~\cite{AAAI2020} conference has 8800 submitted papers. It is crucial to investigate and understand the performance of our model at different scales of hardness. Therefore, we conduct a set of empirical studies with datasets generated by different hyperparameters.

\paragraph{The number of papers $k$ assigned to each reviewer}
Since   reviewers can professionally review a large amount of papers in practice,   we compare the algorithm performance under different $k$,  changing from $3$ to $7$, as shown in the first column of Figure~\ref{fig:paper_per}. We observe improved performance as the number $k$ of papers per reviewer increases.  Interestingly, it turns out that $k=6$ serves as a sweet spot to balance reviewers' workload and the calibration performance since after $k=6$ the calibration quality starts to increases only mildly. Our model LSC (linear) can always perfectly recover the true qualities of the best papers under the noiseless case even when each reviewer only reviews 3 papers.  

\begin{figure}
	\includegraphics[width=0.5\linewidth]{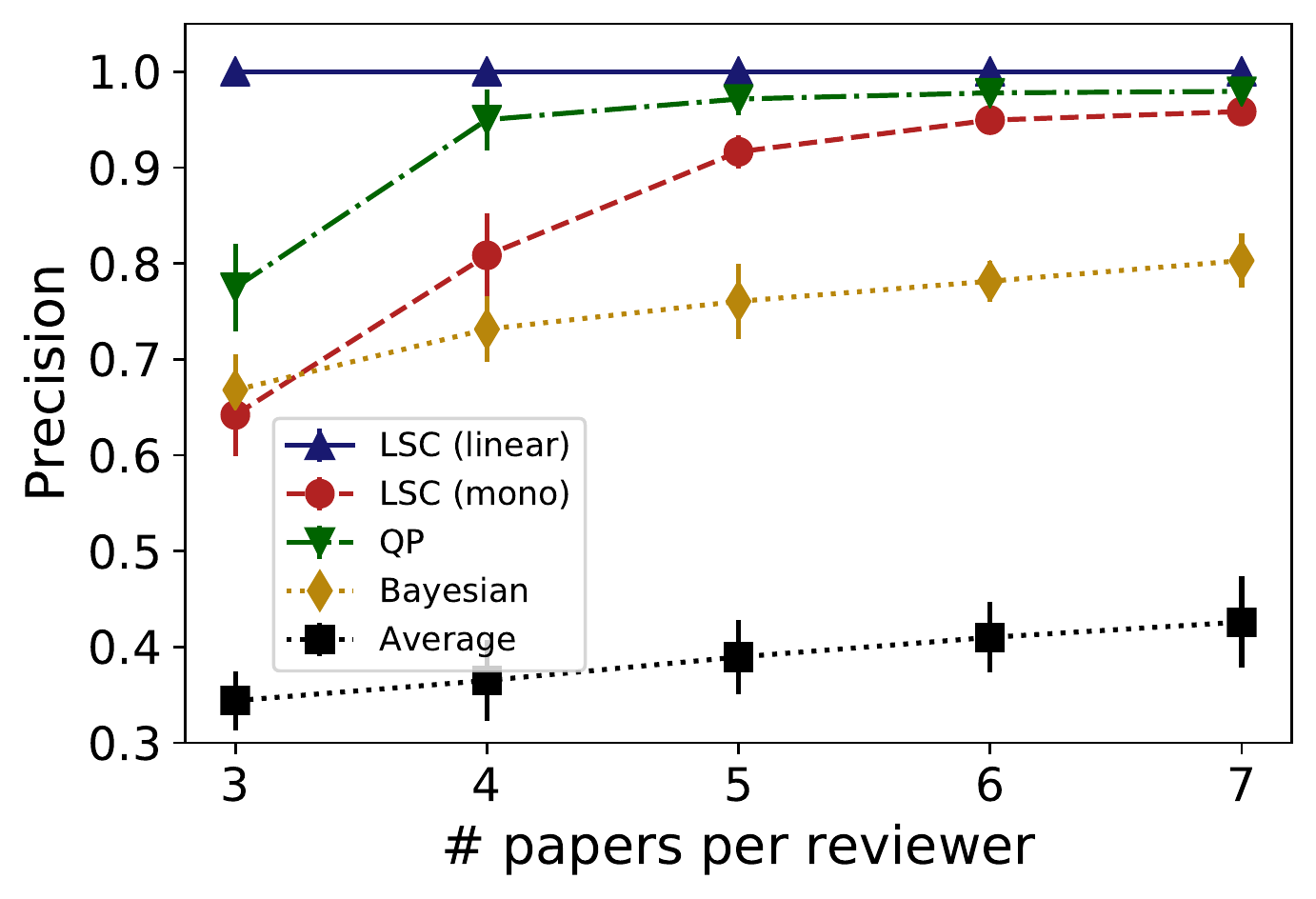}
	\includegraphics[width=0.5\linewidth]{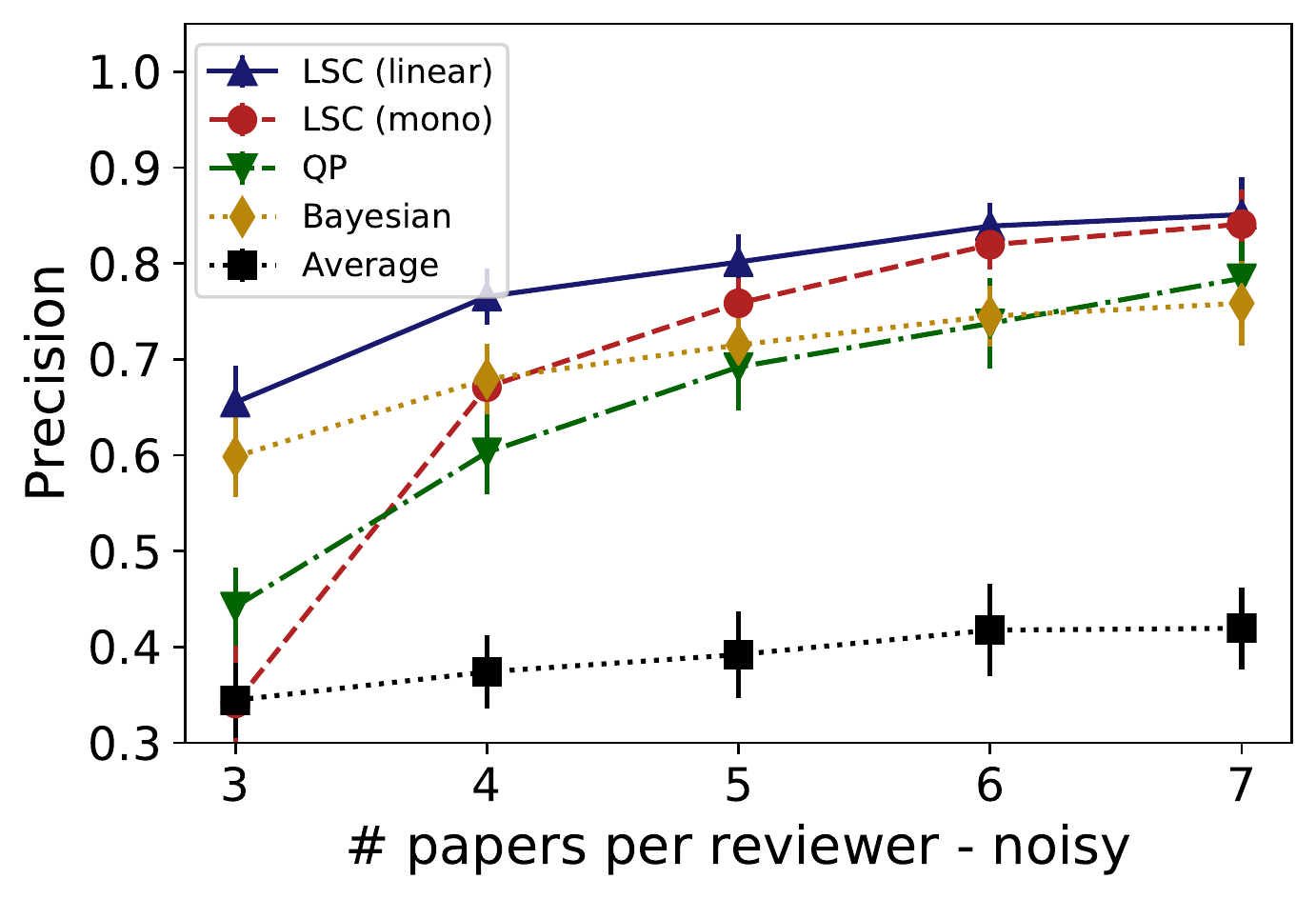}
	\caption{
		The experiments evaluate the performance of baseline models as well as the LSC under monotone or linear constraints. The first plot shows how the algorithms' performance changes according to $k$, \textbf{the number of papers assigned to each reviewer} in the \textbf{noiseless} $(\sigma=0)$ setting, the second plot shows the performance change in the \textbf{noisy} $(\sigma=0.5)$ setting. 	}
	\label{fig:paper_per}
\end{figure}

\paragraph{The paper to reviewer ratio ($N : M$)}
While we assume a $1:1$ ratio between the number of papers and reviewers in our standard setup, there is sometimes less reviewers than papers in large academic conferences.\footnote{For example, NeurIPS-2019~\cite{nips2019} receives $6743$ submissions and has  around $4500$ reviewers, leading to a paper to reviewer ratio  around $3:2$.} To understand the algorithm performance under different paper to reviewer ratios, we test a $3:2$ ratio and $2:1$ ratio by fixing other hyperparameters (e.g., $k=5, N=1000$). As can be seen in Figure~\ref{fig:ratio},  comparing to the performance drop with a smaller $k$ in the Figure~\ref{fig:paper_per}, the decrease in the number of reviewer has similar effect on our models, but is a relatively less important factor for the performance. Our model LSC (linear) can always perfectly recover the true qualities of the best papers under the noiseless case even when the reviewer to paper ratio is $1:2$.


\begin{figure}
	\includegraphics[width=0.5\linewidth]{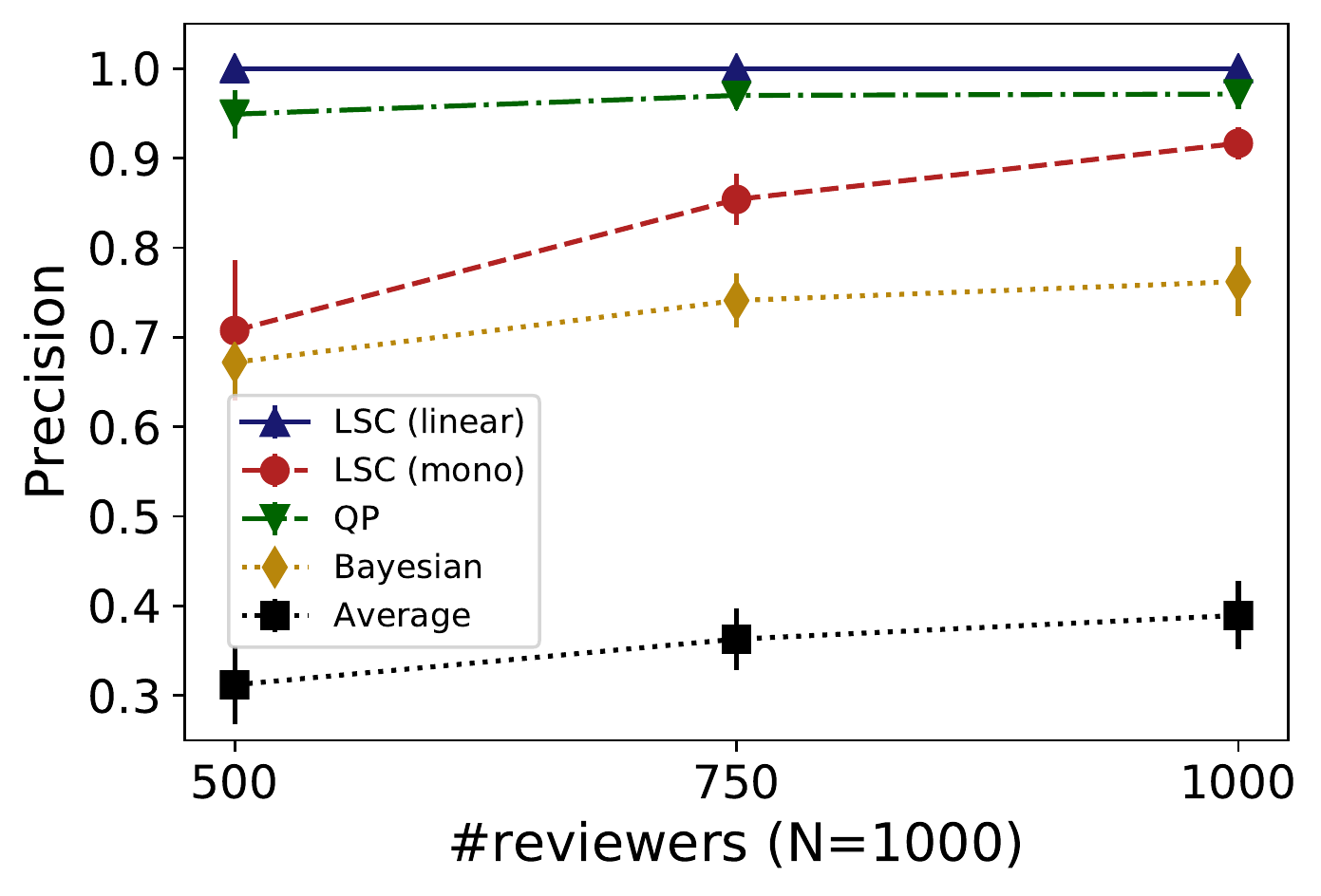}
	\includegraphics[width=0.5\linewidth]{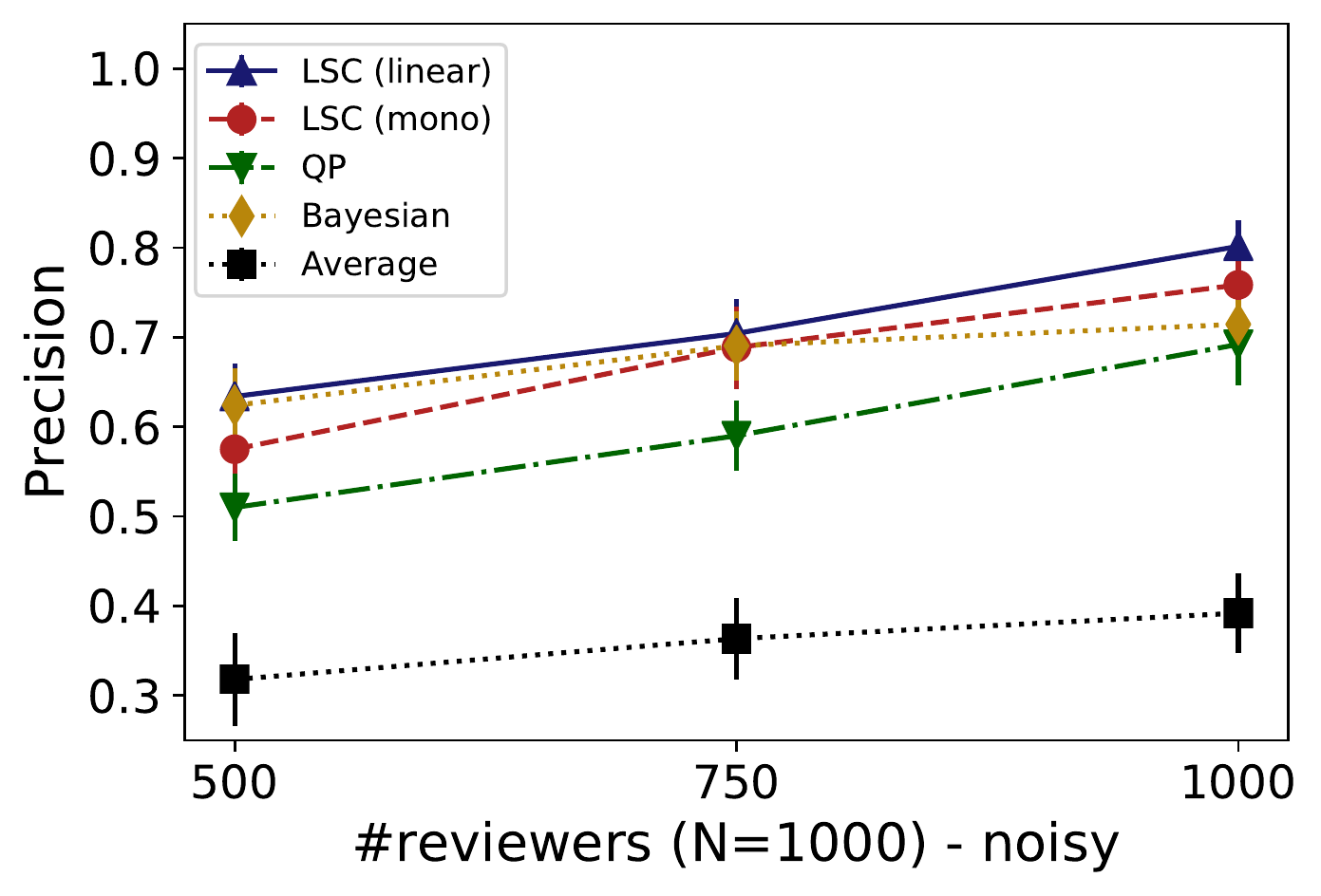}
	\caption{
		The experiments evaluate the performance of baseline models as well as the LSC under monotone or linear constraints. The first plot shows how the algorithms' performance changes according to $N : M$, \textbf{the paper to reviewer ratio} in the \textbf{noiseless} $(\sigma=0)$ setting, the second plot shows the performance change in the \textbf{noisy} $(\sigma=0.5)$ setting.
	}
	\label{fig:ratio}
\end{figure}



\paragraph{The noise scale $\sigma$}
All the previous empirical studies are to understand the impact from the structure of the \emph{review graph}. Finally, we explore a different dimension that adds difficulty to calibration, i.e., the noise scale of perception. 
Results are shown in Figure~\ref{fig:noise-level}. We can see that while the performance our proposed calibration model degrades gracefully as the noise scale increases, it steadily outperforms all the baseline methods. Among them, the Bayesian model is most robust against the noise, while the QP algorithm is very sensitive to noise. Moreover, as noise scale increases, the advantage of  LSC (linear) over LSC (mono) gradually disappears. This is because large noise level makes the prior knowledge of linear scoring function less useful, as we have pointed out in Appendix \ref{append:prior}. 


\begin{figure}
	\centering
	\includegraphics[width=0.5\linewidth]{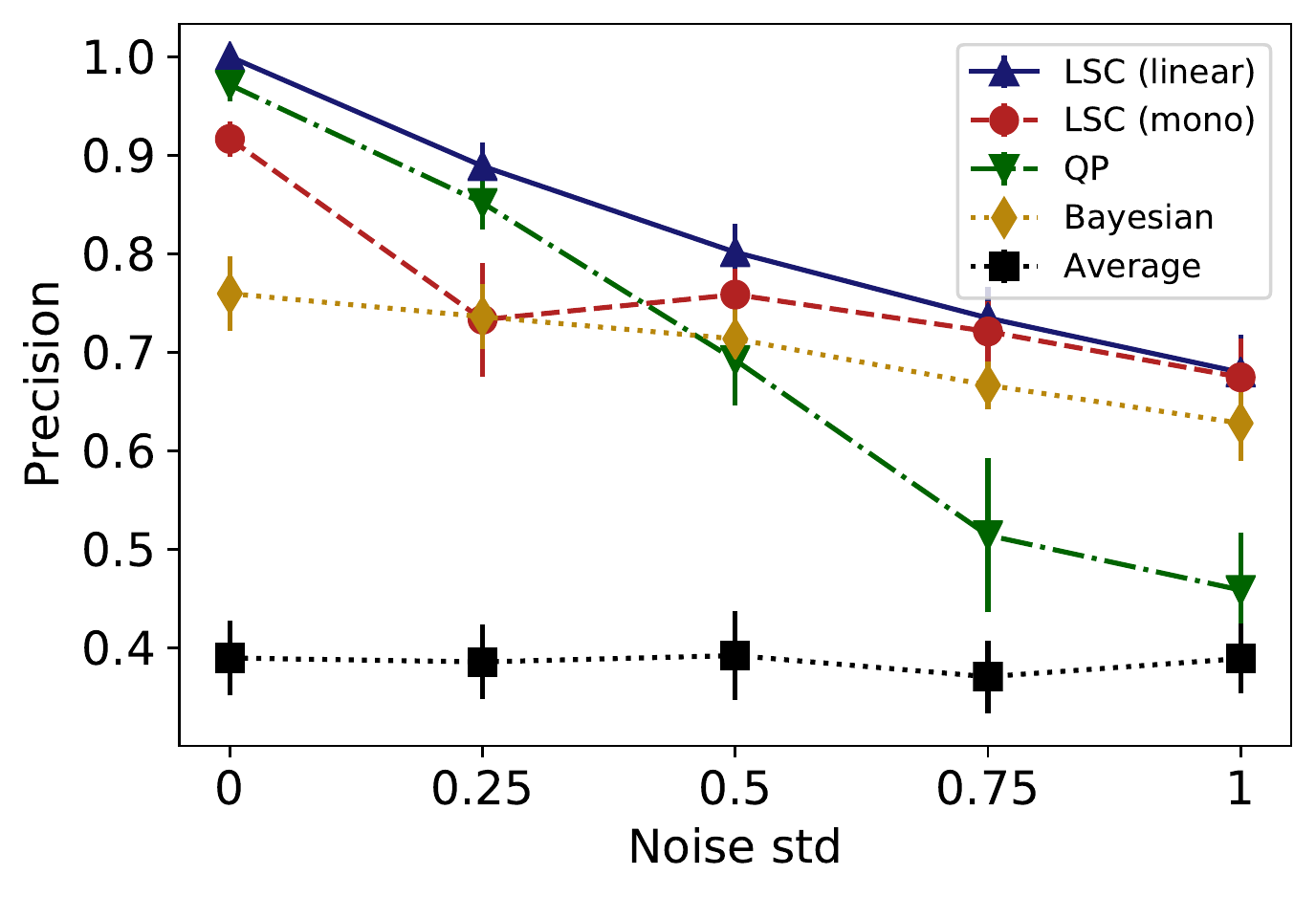}
	\caption{
		The experiments evaluate the performance of baseline models as well as the LSC under monotone or linear constraints. The plot shows how the algorithms' performance changes according to $sigma$, \textbf{the noise scale}.
	}
	\label{fig:noise-level}
\end{figure}

\subsection{Performance Comparison in Rank-Aware Metrics}
In this section, we further investigate the quality of our calibration through rank-aware metrics, that are typically used in the evaluation of information retrieval and recommendation system.
Specially, we consider the following metrics that are designed to measure different aspects of ranking:

\begin{itemize}[leftmargin=*]
	\item \textbf{Average L1} $\omega(\mathbf{x};\mathbf{x}^*) = \frac{1}{N}\sum_{i \in [N]}| \textup{Rank}(i|\mathbf{x}^*) - \textup{Rank}(i|\mathbf{x}) |$ where $\textup{Rank}(i|\mathbf{x})$ is the rank of paper $i$ under scores $\mathbf{x}$. It measures the   L1 distance of each item's ranking given respectively by the true qualities and the recovered qualities. It is used to quantify how close the recovered ranking is to the true ranking. The smaller this distance is, the closer the recovered ranking is to the ranking of true qualities. 
	
	\item \textbf{Average Precision (AP)} $\phi(S, T) = \frac{1}{p(|S|)}\sum_{i\in |S|} p(i)\cdot \one[i \in T] $ where $p(i)$ is the number of papers that are in $T$ and ranked at least as high as $i$ in $S$. It measures if the items in $T$ are indeed recovered with relatively high quality by the model as the accepted papers. The larger this metric is, the better.

	\item \textbf{Normalized Discounted Cumulative Gain (NDCG)} $\psi(S,T) = \frac{\sum_{i\in |S|} \log^{-1} (\textup{Rank}(i|\mathbf{x})+1) \cdot \one[i \in T] }{\sum_{i\in |S|} \log^{-1} (\textup{Rank}(i|\mathbf{x})+1) }$. Similar to AP, it measures the ranking quality of the items that are both in $S$ and $T$. In particular, the smooth logarithmic discounting factor has a good theoretical basis~\cite{wang2013theoretical}. The larger this metric is, the more top papers in $T$ are recovered with high qualities in $S$.
\end{itemize}
Note that the average L1 metric does not apply to Bayesian model, because it is designed to estimate the likelihood of each paper getting accepted, and thus cannot provide reasonable quality estimation of the unaccepted papers.
In Table \ref{tab:rank}, we can see that our proposed LSC model with linear constraint still have the best performance in the three different ranking metrics. Meanwhile, the performance of QP in ranking metrics under the noiseless setting is close to perfect, which suggests that this baseline only accepts very few papers that should be rejected and ranked them relatively low among the accepted papers. The Bayesian model however is more robust, as its performance drops the least from noiseless to noisy setting. 

\begin{table*}[ht!]
	\begin{center}
	\begin{tabular}{ |c | c c c || c c c|}
	\hline
	\diagbox{Model}{Metric} & Avg. L1 & AP (\%) & NDCG (\%)  & Avg. L1 & AP (\%) & NDCG (\%)  \\
	\hline
	Average & 209.7 $\pm$ 6.2 & 60.8 $\pm$ 7.9 & 45.6 $\pm$ 4.4 & 206.6 $\pm$ 6.6 & 61.1 $\pm$ 6.6 & 45.8 $\pm$ 4.4\\
	QP   &  5.5 $\pm$ 2.2 & 99.8 $\pm$ 0.4 & 97.9 $\pm$ 1.2 &  79.7 $\pm$ 14.2 & 74.2 $\pm$ 9.0 & 68.9 $\pm$ 6.9\\
	Bayesian   &  N/A & 87.1 $\pm$ 5.7 & 75.9 $\pm$ 4.3 &  N/A & 83.2 $\pm$ 5.6 & 71.4 $\pm$ 4.0\\
	LSC (mono) & 34.4 $\pm$ 3.0  & 99.4 $\pm$ 0.3 & 93.9 $\pm$ 1.4 & 69.0 $\pm$ 2.8 & 89.1 $\pm$ 3.7 & 79.2 $\pm$ 2.4\\
	LSC (linear) & \textbf{0 $\pm$ 0} & \textbf{100 $\pm$ 0} & \textbf{100 $\pm$ 0} & \textbf{54.5 $\pm$ 2.0} & \textbf{95.4 $\pm$ 1.6} & \textbf{84.7 $\pm$ 2.3} \\
	\hline
	\end{tabular} \\
	\end{center}
	\caption{Experimental results for \textbf{linear scoring functions} setting on the Average, QP~\cite{Roos11}, Bayesian~\cite{Ge11}, and LSC with monotone, and  linear constraints. Each entry contains the mean and standard deviation over 20 trials. The table on the left and right side respectively shows the results in the \textbf{noiseless} ($\sigma=0$) and \textbf{noisy} ($\sigma=0.5$) setting. }
	\label{tab:rank}
\end{table*}

\subsection{More results on Peer-Grading Dataset}\label{append:peer-grading}
In this section, we include more results on the Peer-Grading Dataset. Each homework has about 250 submissions and 200 student reviewers; each submission have at least 6 reviews. However, due to the different settings and difficulty levels of each homework, the calibration results are slightly different. We here include all of these results in Figure \ref{fig:peer-grading}.

We can see that the performance of average baseline is generally good in the Precision metric, but bad in the ranking metric. This suggests it is important to dig into the ranking metric on how the recovered quality can match the ground truth order. Each baseline seems to have certain scenario that they can top the remaining models. And this potentially means our models designed for peer reviews may not be sufficient to model all the factors in the peer grading tasks. However, to our best effort, this is the only real-world dataset to test our model performance beyond synthesized dataset. And it is fair to conclude that \LSC (linear) does show the overall best and most stable calibration performance in different ranking metrics for different $n$. 

\begin{figure}
\includegraphics[width=\textwidth]{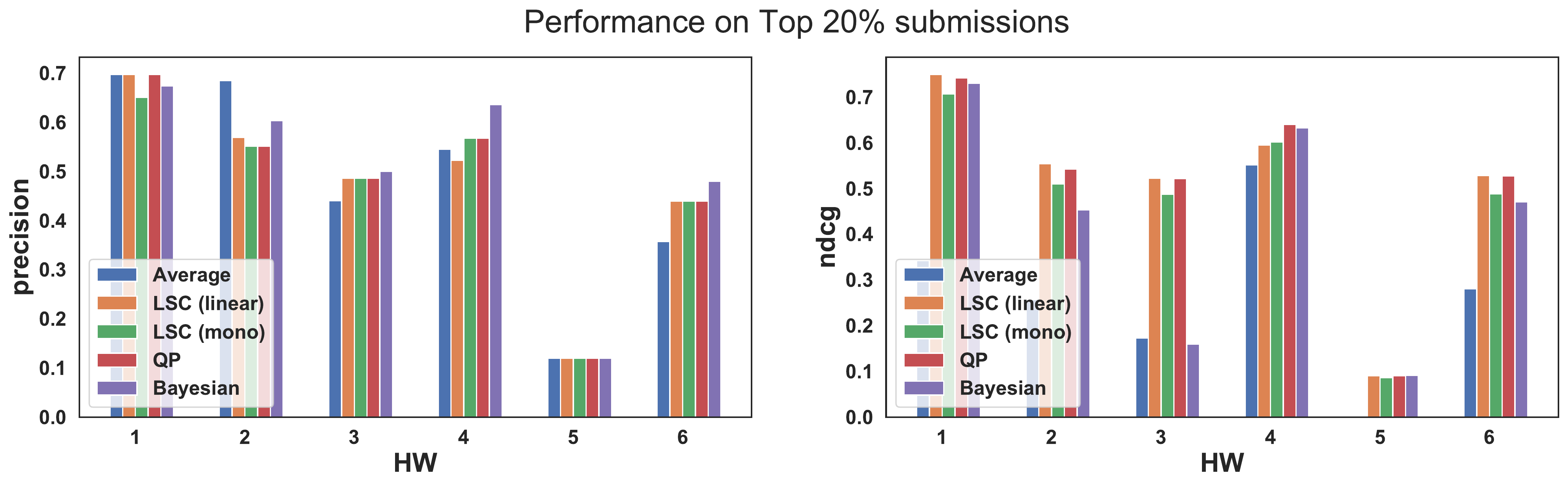}	
\includegraphics[width=\textwidth]{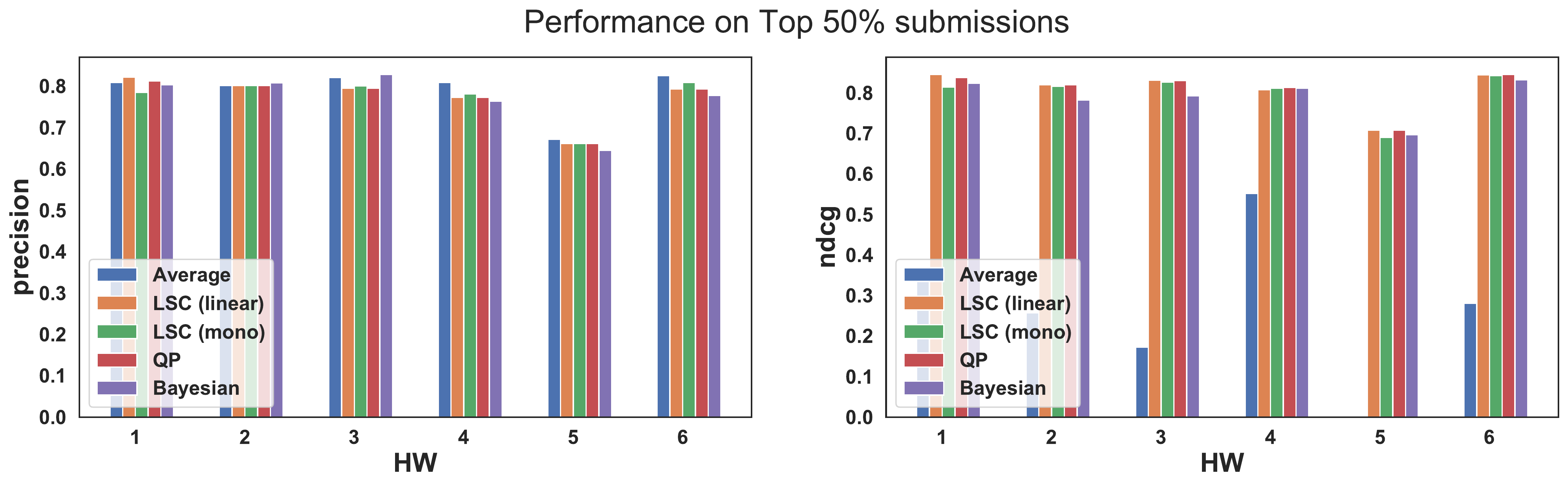}	
\caption{Experimental results on Peer-Grading dataset. In each plot, the performance of each model is compared in six different homework. $n$ is set as the 20\%, 50\% of total papers to be selected in the plots of each row from top to bottom. We use the metric, Precision, NDCG in the plots of each column from left to right.    }
\label{fig:peer-grading}
\end{figure}

\end{document}